\documentclass[11pt]{article}

\usepackage[top=1.25in, bottom=1.25in, left=1.25in, right=1.25in]{geometry}
\usepackage[utf8]{inputenc} 
\usepackage[T1]{fontenc}    
\usepackage{hyperref}       
\usepackage{url}            
\usepackage{booktabs}       
\usepackage{amsfonts}       
\usepackage{nicefrac}       
\usepackage{microtype}      

\usepackage{array}
\newcolumntype{I}{!{\vrule width 3pt}}
\newlength\savedwidth

\newlength\savewidth

\usepackage{amsthm,amsmath,amssymb}
\usepackage{color,graphicx}
\usepackage{multicol}
\usepackage{multirow}
\usepackage{caption}
\usepackage{subcaption}
\usepackage[boxruled,linesnumbered]{algorithm2e}
\usepackage{algorithmic}
\usepackage{enumitem}
\usepackage{url}
\usepackage{hhline}

\renewcommand{\epsilon}{\varepsilon}

\newcommand{\event}{\mathcal{A}}
\newcommand{\eventb}{\mathcal{B}}

\newcommand{\eps}{\varepsilon}

\usepackage{mhequ}
\def \be{\begin{equs}}
\def \ee{\end{equs}}

\newtheorem{theorem}{Theorem}[section]
\newtheorem{lemma}[theorem]{Lemma}
\newtheorem{definition}[theorem]{Definition}

\newtheorem{corollary}[theorem]{Corollary}

\newtheorem{remark}[theorem]{Remark}

\newcommand{\CS}{\mathsf{CS}}
\newcommand{\DI}{\rho_q}
\newcommand{\MD}{\delta_q}

\newcommand{\eat}[1]{}

\newcommand{\I}{\mathrm{I}}

\newcommand{\R}{\mathbb{R}}

\newcommand{\Exp}{\mathbb{E}}

\newcommand{\calD}{\mathcal{D}}
\newcommand{\calE}{\mathcal{E}}
\newcommand{\calF}{\mathcal{F}}

\newcommand{\calQ}{\mathcal{Q}}

\newcommand{\calX}{\mathcal{X}}

\title{\bf Classification with Fairness Constraints: \\ A Meta-Algorithm  with Provable Guarantees}

\author{L. Elisa Celis, Lingxiao Huang, Vijay Keswani and Nisheeth K. Vishnoi \\ EPFL, Switzerland}

\date{}

\allowdisplaybreaks[4]

\begin{document}

\maketitle

\begin{abstract}
Developing classification algorithms that are fair with respect to sensitive attributes of the data has become an important problem due to the growing deployment of classification algorithms in various social contexts.
Several recent works have focused on fairness with respect to  a specific metric, modeled the corresponding fair classification problem as a constrained optimization problem, and developed tailored algorithms to solve them.
Despite this, there still remain important  metrics for which we do not have fair classifiers and  
many of the aforementioned algorithms do not come with theoretical guarantees; perhaps because the resulting optimization problem is non-convex.
The main contribution of this paper is a new meta-algorithm for classification that takes as input a large class of  fairness constraints, with respect to multiple non-disjoint sensitive attributes, and which comes with provable guarantees.
This is achieved by first developing a meta-algorithm for a large family of classification problems with convex constraints, and then showing that classification problems with general types of fairness constraints can be reduced to those in this family.
We present empirical results that show that our algorithm can achieve near-perfect fairness with respect to various fairness metrics, and that the loss in accuracy due to the imposed fairness constraints is often small.
Overall, this work unifies several prior works on fair classification, presents a practical algorithm with theoretical guarantees, and can handle fairness metrics that were previously not possible.
\end{abstract}

\newpage

\tableofcontents

\newpage
\section{Introduction}
\label{sec:introduction}

Classification algorithms are increasingly being used in many societal contexts such as  criminal recidivism~\cite{northpointe2012compas}, predictive policing~\cite{hvistendahl2016can}, and job screening~\cite{miller2015can}.
There are growing concerns that these algorithms may introduce significant bias with respect to certain sensitive attributes, e.g., against black people while predicting future criminals~\cite{flores2016false,angwin2016machine,berk2009role}, granting loans~\cite{dedman1988color} or NYPD stop-and-frisk~\cite{goel2016precinct}, and against women while recommending jobs~\cite{datta2015automated}.
The US Executive Office \cite{united2014big} also voiced concerns about discrimination in automated decision making, including health care delivery and education.
Further, introducing bias may be illegal due to anti-discrimination laws~\cite{act1964civil,magarey2004sex,barocas2016big}, or can create social imbalance~\cite{executive2016big,acm2017statement}.
Thus, developing classification algorithms that can be fair with respect to sensitive attributes  has become an important problem in machine learning.

In classification, one is given a data vector and the goal is to decide whether it satisfies a certain property.
The algorithm is allowed to learn from a set of labeled data vectors that may be assumed to come from a distribution.
The accuracy of a classifier is measured as the probability that the classifier correctly predicts the label of a data vector drawn from the same distribution.
Each data vector, however, may also have a small number of sensitive attributes such as race, gender, and political opinion, and each setting of a sensitive attribute gives rise to potentially non-disjoint groups of data points.
Since fairness could mean different things in different contexts, a number of metrics have been used to determine how fair a classifier is with respect to a sensitive group when compared to another, e.g., statistical parity~\cite{dwork2012fairness}, equalized odds~\cite{hardt2016equality}, and predictive parity~\cite{dieterich2016compas}.
To ensure fairness across different groups, one idea is to make predictions without the information of sensitive attributes, which avoids disparate treatment~\cite{act1964civil}. However, since the learning data may contain historical biases, classifiers trained on such data may still have indirect discrimination for certain sensitive groups~\cite{pedreshi2008discrimination}.

Several recent works use the sensitive attributes and the desired notion of group fairness to place constraints on the classifier -- formulating it as a constrained optimization problem that maximizes accuracy -- and develop tailored algorithms to find such classifiers, e.g., constrained to statistical parity~\cite{zafar2017fairness,menon2018the,goel2018non} or equalized odds~\cite{hardt2016equality,zafar2017fair,menon2018the}.
However, many of these algorithms are without a provable guarantee, perhaps because the resulting optimization problem turns out to be non-convex; e.g., for statistical parity~\cite{zafar2017fairness,krasanakis2018adaptive} and equalized odds~\cite{zafar2017fair}.

Predictive parity, that measures whether the fractions over the class distribution for the predicted labels are close between different groups,
is important in predicting criminal recidivism~\cite{flores2016false,dieterich2016compas},  stopping-and-frisking pedestrians~\cite{goel2016precinct}, and predicting heart condition~\cite{pleiss2017on}.
Dieterich et al.~\cite{dieterich2016compas} investigated such definitions on the dataset \textbf{COMPAS}~\cite{compas}.
Concretely, they check whether the probability of recidivating (positive class), given a high-risk score of a classifier (positive prediction), is similar for blacks and whites, which avoids racial bias.
Similar concerns apply to  the NYPD stop-and-frisk program~\cite{goel2016precinct}, where pedestrians are stopped on the suspicion of carrying illegal weapon.
It may cause racial disparities by having different weapon discovery rates for different races.
Another motivating example of false discovery parity is mentioned in~\cite{pleiss2017on}.
They consider the classification problem on the dataset \textbf{Heart} in which the goal is to accurately predict whether or not an individual has a heart condition.
Since a false prediction of a heart condition could result in
unnecessary medical attention, it is desirable to reduce the disparity between men and women that avoids unfair cost for a gender group.
Zafar et al.~\cite{zafar2017fair} listed false discovery/false omission parity (which are two types of predictive parity) as two special measures of disparate mistreatment.
They left it as an open problem to extend their algorithm to solve the fair classification problem with false discovery or false omission parity.

We present a new meta-classification algorithm that takes as input a large class of fairness constraints, with respect to multiple non-disjoint sensitive attributes, and which comes with theoretical guarantees.
This is achieved by first developing a meta-algorithm for a family of fair classification problems with convex constraints and subsequently showing that classification problems with more general constraints can be reduced to those in this family.
We also present empirical results that show that our algorithm is practical and, as in prior work, the loss in accuracy due to fairness constraints is small.
Moreover, the empirical results show that our algorithm can handle predictive parity well.
Overall, this work unifies and extends several prior works on fair classification, and presents a practical algorithm with theoretical guarantees.

We develop a general form of constrained optimization problems which encompasses many existing fair classification problems.
Since the constraints of the general problem can be non-convex, we relax the constraints to be linear (Section~\ref{sec:model}) and present an algorithm for the resulting linear constrained optimization problem (Section~\ref{sec:algorithm}).
We reduce the general problem to solving a family of linear constrained optimization problems (Section~\ref{sec:relaxed_algorithm}).
We present experiments on \textbf{Adult, German credit} and \textbf{COMPAS} dataset that show that our algorithm achieves a reasonable tradeoff between fairness and accuracy and can handle more constraints like predictive parity (Section~\ref{sec:experiment}).

Overall, we propose a general framework to handle many existing fairness definitions, instead of designing specified algorithms for different fairness definitions (Table~\ref{tab:result}).
Our framework comes with provable guarantees that the resulting classifier is approximately optimal.
To the best of our knowledge, our framework is the first one that can handle predictive parity with provable guarantees.

\subsection{Related Work}
\label{sec:related}

The most relevant works to ours in technique are~\cite{davies2017algorithmic,kamishima2012fairness,menon2018the}, all of which considered the Bayesian classification model for statistical parity or equalized odds.
They either reduce their problem to unconstrained optimization problem by the Lagrangian principle or can be alternately expressed in that form.
Our work uses similar techniques but in comparison can handle a wider class of fairness metrics.

Another approach is to propose other fairness metrics as a proxy of statistical parity or equalized odds, e.g.,~\cite{zafar2017fair,zafar2017fairness,goel2018non}.
Zafar et al.~\cite{zafar2017fair,zafar2017fairness} proposed a covariance-type constraint for statistical parity and equalized odds.
The third approach post-processes a baseline classifier by shifting the decision boundary (can be different for different groups), e.g.,~\cite{fish2016confidence,hardt2016equality,goh2016satisfying,pleiss2017on, woodworth2017learning,dwork2018decoupled}.
These approaches modify the constrained classification problem.

There are increasingly many works with provable guarantees, including \cite{fish2016confidence,hardt2016equality,pleiss2017on, woodworth2017learning}, that provide different classification algorithms with constraints on statistical parity or equalized odds, and \cite{joseph2016fairness, kearns2017meritocratic} for fairness in multi-armed bandit settings or ranking problems respectively. 
To the best of our knowledge, our algorithm is the first unifying framework for all current \cite{arvindVideo} and potential future fairness metrics, with provable guarantees.

\cite{agarwal2018reductions,quadrianto2017recycling} also provide a general framework to handle multiple fairness constraints. 
Quadrianto and Sharmanska \cite{quadrianto2017recycling} encode fairness constraints as a distance between the distributions for different values of a single binary sensitive attribute, and then use the privileged learning framework to optimize loss with respect to fairness constraints. 
While this results in an interesting heuristic, they do not provide theoretical guarantees for their approach.
Agarwal et al. \cite{agarwal2018reductions} give a method to compute a nearly optimal fair classifier with respect to linear fairness constraints, like demographic parity or equalized odds, by the Lagrangian method. 
However their framework does not support constraints on predictive parity (see Remark~\ref{remark:agarwalPaper}).

Another line of research is to pre-process on the training data and achieve an unbiased dataset for learning, e.g.,~\cite{kamiran2009classifying,luong2011k,kamiran2012data,zemel2013learning,feldman2015certifying,krasanakis2018adaptive}.
This approach is quite different from ours since we focus on learning classifiers and investigating the accuracy-fairness tradeoff from the feeding dataset.

Beyond group fairness, recent works also proposed other fairness definitions concerned in classification.
Dwork et al.~\cite{dwork2012fairness} and Zemel et al.~\cite{zemel2013learning} discussed a notion of \emph{individual fairness} that similar individuals should be treated similarly.
\cite{zafar2017from} defined \emph{preference fairness} based on the concepts of fair division and envy-freeness in economics.
Moreover, Grgi{\'c}-Hla{\v{c}}a et al.~\cite{grgic2018beyond,grgic2018human} discussed \emph{procedural fairness} that investigates which input features are fair to use in the decision process and how including or excluding the features would affect outcomes.
Finally, Chouldechova~\cite{chouldechova2017fair} and Kleinberg et al.~\cite{kleinberg2017inherent} investigated the inherent tradeoff between equalized odds and predictive parity (called well-calibrated in their papers).

\begin{table}[t]
	\centering
	\caption{Capabilities of different frameworks in handling different fairness metrics. 
		The symbol $\surd$ (or $\star$) represents that the corresponding framework can handle (or can be extended to handle) the corresponding fairness metrics respectively.
	} 
	\scriptsize{
		\begin{tabular}{|*{13}{c|}} \hline
			\multicolumn{2}{|c|}{\multirow{2}{*}{}} & \multicolumn{2}{c|}{$q_i(f)$} &  \multirow{2}{*}{L/LF }
			& \multirow{2}{*}{This}
			&  \multirow{2}{*}{\cite{hardt2016equality}}
			&  \multirow{2}{*}{\cite{woodworth2017learning}}
			& \multirow{2}{*}{\cite{zafar2017fairness} }
			& \multirow{2}{*}{\cite{zafar2017fair} }
			& \multirow{2}{*}{\cite{menon2018the} } 
			& \multirow{2}{*}{\cite{goel2018non}} 
			& \multirow{2}{*}{\cite{krasanakis2018adaptive}} \tabularnewline \cline{3-4}
			\multicolumn{2}{|c|}{} & $\calE$ & $\calE'$ & & & & & & & & & \tabularnewline \cline{1-13}
			\multirow{13}{*}{\rotatebox{90}{fairness defn.}} &
			statistical & $f=1$ & $\emptyset$ & $\calQ_{\mathrm{lin} }$& $\surd$ & & &   $\surd$& & $\surd$ & $\surd$ & $\surd$ \tabularnewline \cline{2-13}
			& conditional statistical & $f=1$ & $X\in S$ & $\calQ_{\mathrm{lin} }$& $\surd$ & & &   $\surd$& & $\star$ & $\star$ & \tabularnewline \cline{2-13}
			& false positive & $f=1$ & $Y=0$ &$\calQ_{\mathrm{lin} }$& $\surd$&  $\surd$&  $\surd$& & $\surd$ & $\star$& & $\surd$ \tabularnewline \cline{2-13}
			& false negative & $f=0$ & $Y=1$ &$\calQ_{\mathrm{lin} }$& $\surd$&  $\star$& $\star$& & $\surd$ & $\star$& & $\surd$ \tabularnewline \cline{2-13}
			& true positive & $f=1$ & $Y=1$ &$\calQ_{\mathrm{lin} }$& $\surd$&  $\surd$& $\surd$&  & $\star$ & $\surd$ & &  \tabularnewline \cline{2-13}
			& true negative & $f=0$ & $Y=0$ &$\calQ_{\mathrm{lin} }$& $\surd$&  $\star$& $\star$& & $\star$ & $\star$& & \tabularnewline \cline{2-13}
			& accuracy & $f=Y$ & $\emptyset$ &$\calQ_{\mathrm{lin} }$& $\surd$& & & &  $\surd$  &  $\star$ & & \tabularnewline \cline{2-13}
			& false discovery & $Y = 0$ & $f=1$ & $\calQ_{\mathrm{linf} }$ & $\surd$ & & & & & & & \tabularnewline \cline{2-13}
			& false omission & $Y = 1$ & $f=0$ &$\calQ_{\mathrm{linf} }$ & $\surd$& & &  & & & & \tabularnewline \cline{2-13}
			& positive predictive & $Y = 1$ & $f=1$ & $\calQ_{\mathrm{linf} }$ & $\surd$ & & & & & & & \tabularnewline \cline{2-13}
			& negative predictive & $Y = 0$ & $f=0$ & $\calQ_{\mathrm{linf} }$ & $\surd$ & & & & & & & \tabularnewline \hline
		\end{tabular}
	}
	\label{tab:result}
\end{table}

\section{Background and Notation}
\label{sec:model}

We consider the Bayesian model for classification.
Let $\Im$ denote a joint distribution over the domain $\calD=\calX\times [p_1]\times \cdots \times [p_n]\times\left\{0,1\right\}$ where $\calX$ is the feature space.
Each sample $(X,Z_1,\ldots,Z_n,Y)$ is drawn from $\Im$ where each $Z_i\in [p_i]$ ($i\in [n]$) represents a sensitive attribute, and $Y\in \left\{0,1\right\}$ is the label of $(X,Z_1,\ldots,Z_n)$ that we want to predict.
For the sake of readability, we discuss the case where there is only one sensitive attribute $Z\in \left\{1,2,\ldots,p\right\}$ in the main text.
We defer the case of multiple sensitive attributes to Appendix~\ref{app:multi}.
Fixing different values of $Z$ partitions the domain $\mathcal{D}$ into $p$ groups 
$$G_i:=\left\{(x,i,y)\in\calD\right\}.$$
Let $\calF$ denote the collection of all possible classifiers.
Given a loss function $L(\cdot,\cdot)$, there are two models of binary classification:
\begin{enumerate}
	\item If $Z$ is not used for prediction, then our goal is to learn a classifier $f:\calX\rightarrow \left\{0,1\right\}$ that minimizes $L(f;\Im)$.
	In this model, $\calF= \left\{0,1\right\}^\calX$.
	\item If $Z$ is used for prediction, then our goal is to learn a classifier $f:\calX\times [p]\rightarrow \left\{0,1\right\}$ that minimizes $L(f;\Im)$.
	In this model, $\calF= \left\{0,1\right\}^{\calX\times [p]}$.
\end{enumerate}
Denote $\Pr_{\Im}[\cdot]$ to be the probability with respect to $\Im$.
If $\Im$ is clear from context, we denote $\Pr_{\Im}[\cdot]$ by $\Pr[\cdot]$.
A commonly used loss function is the prediction error, i.e., 
$$L(f;\Im)=\Pr_\Im\left[f\neq Y\right].$$
Here, by abuse of notation, we use $f$ to represent $f(X)$ for the first model and $f(X,Z)$ for the second model.

\begin{remark}
Zafar et al.~\cite{zafar2017fairness,zafar2017fair} studied the first model.
Since the output is a single classifier for all groups, for any $x\in \calX$ and $i,j \in [p]$, the predictions for $(x,i)$ and $(x,j)$ are the same, i.e., disparate treatment does not happen.
The second model is also investigated in ~\cite{hardt2016equality,menon2018the}.
The goal can be regarded as learning a different classifier for each $G_i$. 
Hence, disparate treatment can arise.
\end{remark}

\noindent
Apart from minimizing the loss function, existing fair classification problems also want to achieve similar \emph{group performances} for all $G_i$.
There are many metrics to measure the group performance, including statistical/true positive/accuracy/false discovery rates; see Table~\ref{tab:result} for a summary. 
For example, the statistical rate of $G_i$ is of the form $\Pr_\Im\left[ f=1\mid G_i\right]$, i.e., the probability of an event ($f=1$) conditioned on another event ($G_i$).
Formally, we define group performance as follows.

\begin{definition}[\textbf{Group performance and group performance function}]
	\label{def:group_benefit}
	Given a classifier $f\in \calF$ and $i\in [p]$, we call $q^\Im_i(f)$ the group performance of $G_i$ if $q^\Im_i(f)=\Pr_\Im\left[\calE\mid G_i,\calE' \right]$ for some events $\calE,\calE'$ that might depend on the choice of $f$.
	Define a group performance function $q^\Im:\calF\rightarrow [0,1]^p$ for any classifier $f\in \calF$ as $q^\Im(f)=\left(q^\Im_1(f),\ldots,q^\Im_p(f)\right)$.
\end{definition}

\noindent
When $\Im$ is clear from context, we denote $q^\Im$ by $q$.
Since we need to measure $q^\Im_i(f)$, we assume the existence of an oracle to answer $\Pr_{\Im}[\calE]$ for any event $\calE$ as per the  context.
At a high level, a classifier $f$ is considered to be fair w.r.t. to $q$ if all $q_i(f)\approx q_j(f)$.
Consider the following examples of $q$.

\begin{enumerate}[leftmargin=*]
	\item For accuracy rate where $\calE:=(f=Y)$ and $\calE':=\emptyset$, i.e., $q_i(f)$ is the accuracy of the classifier on group $G_i$, we can rewrite $q_i(f)$ as follows:
	\[
	q_i(f)=\Pr\left[Y=0\mid G_i\right]+\sum_{j\in \left\{0,1\right\}}(-1)^{j-1}\cdot \Pr\left[Y=j\mid G_i\right]\cdot \Pr\left[f=1\mid G_i, Y=j\right],
	\]
	i.e., a linear combination of conditional probabilities $\Pr\left[f=1\mid G_i, \mathcal{E}_j\right]$ where all events $\mathcal{E}_j$ are independent of the choice of $f$.
	\item For false discovery rate where $\calE:=(Y=0)$ and $\calE':=(f=1)$, i.e., $q_i(f)$ is the prediction error on the sub-group of $G_i$ with positive predicted labels, we can rewrite $q_i(f)$ as follows:
	\[
	q_i(f)=\frac{\Pr\left[Y=0, G_i\right]\cdot \Pr\left[f=1\mid  G_i, Y=0\right]}{\Pr\left[G_i\right]\cdot \Pr\left[ f=1\mid G_i\right]},
	\]
	i.e., the fraction of two conditional probabilities $\Pr\left[f=1\mid G_i, Y=0\right]$ and $\Pr\left[f=1\mid G_i\right]$.
\end{enumerate}

\noindent
In both these examples, $q_i(f)$ can be written in terms of probabilities $\Pr\left[f=1\mid G_i, \cdot \right]$ as either a linear combination, or as a quotient of linear combinations.
Below we define a general class of group performance functions that generalizes these two examples.

\begin{definition}[\textbf{Linear-fractional/linear group performance functions}]
	\label{def:frac_lin}
	A group performance function $q$ is called \textbf{linear-fractional} if for any $f\in \calF$ and $i\in [p]$, $q_i(f)$ can be rewritten as 
	\begin{align}
	\label{eq:fractional}
	q_i(f) = \frac{ \alpha^{(i)}_0+ \sum_{j=1}^{k} \alpha^{(i)}_j\cdot \Pr_\Im\left[f=1\mid G_i,\event^{(i)}_j\right]}{\beta^{(i)}_0+ \sum_{j=1}^{l} \beta^{(i)}_j\cdot \Pr_\Im\left[f=1\mid G_i,\eventb^{(i)}_j\right]}
	\end{align}
	for two integers $k,l\geq 0$, events $\event^{(i)}_1, \ldots, \event^{(i)}_k, \eventb^{(i)}_1,\ldots,\eventb^{(i)}_l$ that are independent of the choice of $f$, and parameters $\alpha^{(i)}_0,\ldots,\alpha^{(i)}_k, \beta^{(i)}_0,\ldots,\beta^{(i)}_l\in \R$ that may depend on $\Im$ but are independent of the choice of $f$. 

	Denote $\calQ_{\mathrm{linf} }$ to be the collection of all linear-fractional group performance functions.
	Specifically, if $l=0$ and $\beta^{(i)}_0= 1$ for all $i\in [p]$, $q$ is said to be \textbf{linear}.
	Denote $\calQ_{\mathrm{lin} }\subseteq \calQ_{\mathrm{linf} }$ to be the collection of all linear group performance functions.
\end{definition}

\noindent
In Appendix~\ref{app:fair_notion}, we show that all $q$ in Table~\ref{tab:result} are linear-fractional and many of them are even linear.

Given group performance functions $q$, we can formulate and study fair classification problems.
A classifier $f$ is said to satisfy $\tau$-rule w.r.t. to $q$ if \[
\DI(f):=\min_{i\in [p]}q_i(f)/ \max_{i\in [p]}q_i(f)\geq \tau;
\] 
see~\cite{feldman2015certifying,zafar2017fair,zafar2017fairness,menon2018the}.
If $\tau$ is close to 1, $f$ is considered to be more fair. 
Assume there are $m$ fractional group performance functions $q^{(1)},\ldots,q^{(m)}\in \calQ_{\mathrm{linf} }$ and $L(f;\Im)=\Pr_\Im\left[f\neq Y\right]$.
Given $\tau_1,\ldots,\tau_m\in [0,1]$, our main objective is to solve the following fair classification program induced by $\DI$, which captures existing constrained optimization problems~\cite{kamiran2009classifying,davies2017algorithmic,menon2018the} as special cases.
\begin{equation} \tag{$\rho$-Fair}
\label{eq:progDI}
\begin{split}
& \min_{f\in \calF} \Pr_\Im\left[f\neq Y\right] \\ 
& s.t.,~ \rho_{q^{(i)}}(f)=\min_{j\in [p]}q^{(i)}_j(f)/ \max_{j\in [p]}q^{(i)}_j(f)\geq \tau_i, ~ \forall i\in [m].
\end{split}
\end{equation}

\begin{remark}
	\label{remark:fair}
	Specifically, if $\tau_i = 1$, the program above computes a classifier $f$ with perfect fairness w.r.t. to $q^{(i)}$. 
	This setting is	well studied in the literature~\cite{calders2010three,dwork2012fairness,feldman2015certifying,hardt2016equality,zafar2017fairness,zafar2017fair,zemel2013learning}.
	However,  perfect fairness  is known to have deficiencies~\cite{chouldechova2017fair, friedler2016possibility,hardt2016equality,kleinberg2017inherent} and, hence,
	prior work considers the relaxed fairness metric $\tau$-rule where $\tau_i<1$.

	Another relaxed fairness metric is defined by 
	\[
	\MD(f):=\min_{i\in [p]}q_i(f)-\max_{i\in [p]}q_i(f).
	\]
	Computing a classifier $f$ such that $\MD(f)\geq \tau$ ($\tau\in [-1,0]$) has also been investigated in the literature~\cite{calders2010three,menon2018the}.
	We refer the reader to a survey~\cite{zliobaite2017measuring} for other relaxed fairness metrics, e.g., AUC and correlation.
\end{remark}

\noindent
Computationally, the constraints of \ref{eq:progDI} are non-convex in general.
To handle this problem, we consider  \emph{linear fairness constraints}, which have been considered in other fundamental algorithmic problems including sampling \cite{FatML,celis2017complexity,celis2018fair}, ranking \cite{celis2018ranking}, voting~\cite{celis2018multiwinner}, and personalization \cite{celis2018an}. 
We  introduce the following program as a subroutine for solving \ref{eq:progDI}.

\begin{definition}[\textbf{Classification with fairness constraints}]
	\label{def:classification_fairness}
	Given $\ell^{(i)}_j,u^{(i)}_j\geq 0$ for all $i\in [m]$ and $j\in [p]$, the fairness constraint for $G_j$ and $q^{(i)}$ is 
	\[
	\ell^{(i)}_j\leq q^{(i)}_j(f)\leq u^{(i)}_j.
	\]
	We consider the following classification problem with fairness constraints:
	\begin{align} \tag{Group-Fair}
	\label{eq:progfair}
	\begin{split}
	&	\min_{f\in \calF} \Pr_\Im\left[f\neq Y\right] \\ 
	& s.t.,~ \ell^{(i)}_j\leq q^{(i)}_j(f) \leq u^{(i)}_j, ~ \forall i\in [m], j\in [p].
	\end{split}
	\end{align}
\end{definition}

\noindent
It is not hard to see that for any feasible classifier $f$ of \ref{eq:progfair} and any $i\in [m]$, $f$ satisfies $\frac{\min_{j\in [p]} \ell^{(i)}_j }{ \max_{j\in [p]} u^{(i)}_j}$-rule w.r.t. to $q^{(i)}$. 
Moreover, since parameters $\ell^{(i)}_j,u^{(i)}_j$ can be non-uniform, \ref{eq:progfair} can treat different groups differently.
In this sense, \ref{eq:progfair} is more flexible than \ref{eq:progDI}.

\begin{remark} \label{remark:agarwalPaper}
Agarwal et al. \cite{agarwal2018reductions} provide a framework for the above problem when the constraints are linear.
In particular, their framework supports fairness constraints that are linearly dependent on the conditional moments of the form $\mathbb{E}[g(\cdot, f) \mid \mathcal{E}]$, where $g$ is a function that depends on the classifier $f$ along with features of the element while $\mathcal{E}$ is an event that does not depend on $f$.  
However linear-fractional constraints cannot be directly represented in this form, since here the event we condition on, $\mathcal{E}$, depends on the classifier $f$, which is why their framework does not support constraints like predictive parity.
\end{remark}

\section{Theoretical Results}
\label{sec:theoretical}

In this section, we present an efficient meta-algorithm to approximately solve \ref{eq:progDI} that comes with provable guarantees (Theorem~\ref{thm:quantification}, Section~\ref{sec:relaxed_algorithm}).
To this end, we show that \ref{eq:progDI} can be efficiently reduced to \ref{eq:progfair} (Section \ref{sec:reduction}).
Subsequently, we show that there exists a polynomial time algorithm that computes an approximately optimal classifier for \ref{eq:progfair} (Section~\ref{sec:algorithm}).
For convenience, we only consider $m=1$ in this section, i.e., there is only one group performance function $q$ and we require $\DI(f)\geq \tau$ for some $\tau\in [0,1]$.
The general case of multiple group performance functions is discussed in Appendix~\ref{app:multi}.

\subsection{Reduction from \ref{eq:progDI} to \ref{eq:progfair}}
\label{sec:reduction}

We first show the generality of \ref{eq:progfair}, i.e., approximately solving \ref{eq:progDI} can be reduced to solving a family of \ref{eq:progfair}. 
A $\beta$-approximate algorithm for \ref{eq:progfair} ($\beta\geq 1$) is an efficient algorithm that computes a feasible classifier with prediction error at most $\beta$ times the optimal prediction error of \ref{eq:progfair}.

\begin{theorem}[\textbf{Reduction from \ref{eq:progDI} to \ref{eq:progfair}}]
	\label{thm:progDI}
	Given $\tau\in [0,1]$, let $f^\star_\tau$ denote an optimal fair classifier for \ref{eq:progDI}.
	Given a $\beta$-approximate algorithm $A$ for \ref{eq:progfair} ($\beta\geq 1$) and any $\eps>0$, there exists an algorithm that applies $A$ at most $\lceil \tau/\eps \rceil$ times and computes a classifier $f\in \calF$ such that 
	\begin{enumerate}
		\item $\Pr\left[f\neq Y\right]\leq \beta\cdot\Pr\left[f^\star_\tau\neq Y\right]$; 
		\item $\min_{i\in [p]}q_i(f)\geq \tau\cdot \max_{i\in [p]}q_i(f)-\eps$.
	\end{enumerate}
\end{theorem}

\begin{proof}
	Let $T:=\lceil \tau/\eps \rceil $. 
	For each $i\in [T]$, denote $a_i := (i-1)\cdot \eps$ and $b_i := i\cdot \eps/\tau$.
	For each $i\in [T]$, we construct a \ref{eq:progfair} program $P_i$  with $\ell_j=a_i$ and $u_j=b_i$ for all $j\in [p]$.
	Then we apply $A$ to compute $f_i\in \calF$ as a solution of $P_i$.
	Among all $f_i$, we output $f$ such that $\Pr\left[f\neq Y\right]$ is minimized.
	Next, we verify that $f$ satisfies the conditions in the theorem.

	Note that $a_i\geq \tau\cdot b_i-\eps$ for each $i\in [T]$. 
	We have 
	\[
	\min_{i\in [p]}q_i(f)\geq \tau\cdot \max_{i\in [p]}q_i(f)-\eps.
	\]
	On the other hand, assume that 
	$$(j-1)\cdot \eps\leq \min_{i\in [p]}q_i(f^\star_\tau) < j\cdot \eps$$ for some $j\in [T]$.
	Since $f^\star_\tau$ is a feasible solution of \ref{eq:progDI}, we have 
	\[
	\max_{i\in [p]}q_i(f^\star_\tau)\leq \frac{\min_{i\in [p]}q_i(f^\star_\tau)}{\tau}< j\cdot \eps/\tau.
	\]
	Hence, $f^\star_\tau$ is a feasible solution of Program $P_j$.
	Finally, by the definitions of $A$ and $f$, we have 
	\[
	\Pr\left[f\neq Y\right]\leq \Pr\left[f_j\neq Y\right]\leq \beta\cdot \Pr\left[f^\star_\tau\neq Y\right].
	\]
\end{proof}

\noindent
The above theorem can be generalized to any loss function instead of the prediction error.
The above reduction still holds for the $m>1$ case. 
The only difference is that we need to apply algorithm $A$ around $\eps^{-m}$ times.
This enables us to simultaneously handle a constant number of fairness requirements; see Appendix~\ref{app:multi} for details.

\subsection{Algorithm for \ref{eq:progfair}}
\label{sec:algorithm}

In this section, we propose an algorithm for \ref{eq:progfair}.
We only state the main result here and defer the details to Sectioin~\ref{app:algorithm}.
For concreteness, we first consider the case that $\calF=\left\{0,1\right\}^{\calX}$ and $q\in \calQ_{\mathrm{lin} }$. 
By Definition~\ref{def:frac_lin}, assume that 
\[
q_i(f)=\alpha^{(i)}_0+ \sum_{j=1}^{k} \alpha^{(i)}_j\cdot \Pr\left[f=1\mid G_i,\event^{(i)}_j\right]
\]
for $f\in \calF$ and $i\in [p]$.

Without fairness constraints, we can prove that \[
f^\star:=\I\left[\Pr\left[Y=1\mid X=x\right]-0.5>0\right]
\] 
is an optimal classifier minimizing the prediction error $\Pr\left[f\neq Y\right]$. 
(Here $I(\cdot)$ is the indictor function.)
But such $f^\star$ might not satisfy all the fairness constraints as we want.
Hence, we introduce a regularization parameter $\lambda\in \R^p$ and consider the following fairness-aware classification problem
\begin{align}
\label{eq:fairness_aware}
f_\lambda^\star:=\arg\min_{f\in \calF} \Pr\left[f\neq Y\right]-\sum_{i\in [p]} \lambda_i\cdot q_i(f).
\end{align}

\noindent
Now we can ``control'' $q_i(f_\lambda^\star)$ by adjusting $\lambda$.
Intuitively, increasing $\lambda_i$ leads to an increase in $q_i(f_\lambda^\star)$.
By selecting suitable $\lambda$, we can expect that $f^\star_\lambda$ satisfies all fairness constraints.
As we then show there exists some $\lambda\in \R^p$ such that \ref{eq:progfair} is equivalent to \eqref{eq:fairness_aware}, by the Lagrangian principle.
Moreover, $f_\lambda^\star$ can be shown to be an instance-dependent threshold function with the threshold
\begin{align}
\label{eq:s_lambda}
\begin{split}
s_\lambda(x):= \Pr\left[Y=1\mid X=x\right]-0.5+\sum_{i\in [p]}\lambda_i\cdot \psi_i(x),
\end{split}
\end{align}
where 
\[
\psi_i(x)=\sum_{j=1}^{k} \frac{\alpha^{(i)}_j}{\Pr\left[G_i,\event^{(i)}_j \right]}\cdot \Pr\left[G_i,\event^{(i)}_j\mid X=x\right].
\]
Observe that the term $\Pr\left[Y=1\mid X=x\right]-0.5$ is exactly the threshold for the unconstrained optimal classifier $f^\star$, and the remaining term $\sum_{i\in [p]}\lambda_i\cdot \psi_i(x)$ can be regarded as a threshold correction induced by $\lambda$.
Such an approach is also used in other contexts to reduce constrained optimization problems to unconstrained ones, see e.g.,~\cite{maji2011biased,mahoney2012a}.

For a number $t\in \R$, we define $t_+:=\max\left\{0,t\right\}$. 
We summarize the main theorem as follows.

\begin{theorem}[\textbf{Solution characterization and computation for $q\in \calQ_{\mathrm{lin} }$}]
	\label{thm:not_linear}
	Given any parameters $\ell, u\in [0,1]^p$, there exist optimal Lagrangian parameters $\lambda^\star\in \R^p$ such that $\I[s_{\lambda^\star}(X)>0]$ is an optimal fair classifier for \ref{eq:progfair}.
	Moreover, $\lambda^\star$ can be computed in polynomial time as a solution to the following convex program:
	\begin{align}
	\label{eq:main_lambda}
	\lambda^\star = \arg\min_{\lambda\in \R^p} \Exp_X\left[ (s_\lambda(X))_+ \right] + \sum_{i\in [p]} \left(\alpha^{(i)}_0-u_i\right)\lambda_i +\sum_{i\in [p]} \left(u_i-\ell_i\right)\cdot (\lambda_i)_+.
	\end{align}
\end{theorem}

\noindent
The proof of this theorem reduces \ref{eq:progfair} to an unconstrained optimization problem by the Lagrangian principle (Appendix~\ref{sec:unconstrained}). 
Then we derive~\eqref{eq:main_lambda} as the dual program to \ref{eq:progfair} and show that $\lambda^\star$ is an optimal solution to~\eqref{eq:main_lambda} (Appendix~\ref{sec:computation}).
Since $u_i-\ell_i\geq 0$, Program \eqref{eq:main_lambda} is convex and hence we can apply standard convex optimization algorithms, e.g., the stochastic subgradient method~\cite{boyd2008stochastic}.
Consequently, Theorem~\ref{thm:not_linear} leads to an algorihm \ref{eq:progfair}$\left(\Im,q^{\Im}, \left\{\ell_i\right\},\left\{u_i\right\}\right)$ that computes an optimal fair classifier for \ref{eq:progfair}.
Theorem~\ref{thm:not_linear} can also be directly extended to $\calF=\left\{0,1\right\}^{\calX\times [p]}$ by replacing $X$ to $(X,Z)$ everywhere.
Our algorithm is similar to~\cite[Algorithm 1]{menon2018the}, though they focus on the statistical rate and true positive rate.
The paper~\cite{menon2018the} analyzed the characterization but did not show how to compute the optimal Lagrangian parameters.
Our approach can be naturally extended to their setting and to compute the optimal Lagrangian parameters in their framework; see Appendix~\ref{app:another_algorithm} for details.

\begin{remark}
	Theorem~\ref{thm:not_linear} can be generalized to $q\in \calQ_{\mathrm{linf} }$. 
	The key observation is that we can rewrite the fairness constraint $\ell_i\leq q_i(f)$ as
	\[
	\alpha^{(i)}_0+ \sum_{j=1}^{k} \alpha^{(i)}_j\cdot \Pr_\Im\left[f=1\mid G_i,\event^{(i)}_j\right]\geq \ell_i\cdot \left(\beta^{(i)}_0+ \sum_{j=1}^{l} \beta^{(i)}_j\cdot \Pr_\Im\left[f=1\mid G_i,\eventb^{(i)}_j\right]\right).
	\] 
	By rearranging, the above inequality is expressible as a linear constraint $a^\top f+b\leq 0$.
	This also holds for $q_i(f)\leq u_i$, which implies that \ref{eq:progfair} is a linear program of $f$.
	Hence, introducing fairness constraints can handle predictive parity with $q\in \calQ_{\mathrm{linf} }$, but the prior work can not -- due to the fact that the constraint $q_i(f)\geq \tau\cdot q_j(f)$ may not be convex in general.
\end{remark}

\subsection{Algorithm for \ref{eq:progDI}}
\label{sec:relaxed_algorithm}

We proceed to designing an algorithm that handles the fairness metric $\DI$. 
In real-world settings, instead of knowing $\Im$, we only have $N$ samples $\left\{(x_i,z_i,y_i)\right\}_{i\in [N]}$ drawn from $\Im$.
To handle this, we use the idea inspired by~\cite{narasimhan2014on,menon2018the}: estimate $\Im$ by $\widehat{\Im}$, e.g., via Gaussian Naive Bayes or logistic regression on samples, and then compute a classifier based on $\widehat{\Im}$ by solving a family of \ref{eq:progfair} programs as stated in Theorem~\ref{thm:progDI}; see Algorithm~\ref{alg:meta_mult}.

\paragraph{Analyzing Algorithm~\ref{alg:meta_mult}.}

Intuitively, if $\widehat{\Im}$ is close to $\Im$, then the quality of $f$ in both accuracy and fairness should be comparable to an optimal fair classifier for \ref{eq:progDI} under $\Im$.
Define 
\[
\kappa:=2 \max_{i\in [p], f\in\calF}\left|q^{\widehat{\Im}}_i(f)-q^\Im_i(f) \right|
\] 
as the error introduced in $q^\Im$ when replacing $\Im$ by $\widehat{\Im}$.
Let $d_{TV}(\Im, \widehat{\Im})$ denote the total variation distance between $\widehat{\Im}$ and $\Im$.

\newlength{\textfloatsepsave} 
\setlength{\textfloatsep}{0pt}{
	
	\IncMargin{1em}
	\begin{algorithm}[htb]
		\caption{\textrm{A meta-algorithm for \ref{eq:progDI} }}
		\label{alg:meta_mult}
		\SetKwInOut{Input}{Input}
		\SetKwInOut{Output}{Output}
		
		\Indm
		\Input{Samples $\left\{(x_i,z_i,y_i)\right\}_{i\in [N]}$ from distribution $\Im$, a linear-fractional group performance function $q^\Im\in \calQ_{\mathrm{linf} }$, a fairness parameter $\tau\in [0,1]$ and an error parameter $\eps\in [0,1]$.}
		\Output{A classifier $f\in \calF$. } 
		\Indp
		\BlankLine
		Compute an estimated distribution $\widehat{\Im}$ (e.g., via Gaussian Naive Bayes) on $\left\{(x_i,z_i,y_i)\right\}_{i\in [N]}$. \\
		$T\leftarrow \lceil \tau/\eps \rceil$. 
		For each $i\in [T]$, $a_i \leftarrow (i-1)\cdot \eps$ and $b_i \leftarrow i\cdot \eps/\tau$. \\
		For each $i\in [T]$, let $f_i\leftarrow$ \ref{eq:progfair}$\left(\widehat{\Im},q^{\widehat{\Im} },\left\{\ell_j=a_i\right\}_{j\in [p]}, \left\{u_j=b_i\right\}_{j\in [p]}\right)$.  \\
		Return $f\leftarrow \arg\min_{f_i} \Pr_{\widehat{\Im} }\left[f_i \neq Y\right]$.
	\end{algorithm}
	\DecMargin{1em}
	
}

\begin{theorem}[\textbf{Quantification of the output classifier}]
	\label{thm:quantification}
	Let $f^\star$ be a fair classifier minimizing the prediction error $\Pr_\Im\left[f\neq Y\right]$ subject to the relaxed $\tau$-rule: 
	\[
	\min_{i\in [p]} q^\Im_i(f)\geq \tau\cdot \max_{i\in [p]} q^\Im_i(f)+\kappa.
	\]
	Then Algorithm~\ref{alg:meta_mult} outputs a classifier $f$ such that 
	\begin{enumerate}
		\item $\Pr_\Im\left[f\neq Y\right]\leq \Pr_\Im\left[ f^\star\neq Y\right]+2\cdot d_{TV}(\widehat{\Im},\Im) $; 
		\item $\min_{i\in [p]} q^\Im_i(f)\geq \tau \cdot \max_{i\in [p]} q^\Im_i(f)-\eps-\kappa$.
	\end{enumerate}
\end{theorem}

\noindent
The key is to show 
$f^\star$ is feasible for \ref{eq:progDI} under $\widehat{\Im}$ 
and then prove by Theorem~\ref{thm:progDI} that 
\begin{itemize}
\item $\Pr_{\widehat{\Im}}\left[f\neq Y\right]\leq \Pr_{\widehat{\Im}}\left[f^\star\neq Y\right]$;
\item $\min_{i\in [p]} q^{\widehat{\Im}}_i(f)\geq \tau \cdot \max_{i\in [p]} q^\Im_i(f)-\eps$.
\end{itemize}
To account for the error when going from $\widehat{\Im}$ to $\Im$, the terms $2\cdot d_{TV}(\widehat{\Im},\Im)$ and $\kappa$ are introduced.

Theorem~\ref{thm:quantification} quantifies the loss we incur if the estimated distribution $\hat{\Im}$ is not a good fit for the samples. Note that $f^\star$ is only an approximately optimal fair classifier for \ref{eq:progDI} due to the additional error $\kappa$.
Since we do not have access to $\Im$ (only to $\widehat{\Im}$), we cannot compare the output $f$ to the optimal solution of \ref{eq:progDI}, but only to $f^\star$.
If the number of samples $N$ is large, we can expect that $\widehat{\Im}$ and $\Im$ are close, and hence $\kappa, d_{TV}(\widehat{\Im},\Im)$ are small.
Then the performance of $f$ is close to $f^\star$ over $\Im$.
Specifically, if $\widehat{\Im}=\Im$, we have $\kappa=d_{TV}(\widehat{\Im},\Im)=0$.
Then the output $f$ satisfies the properties of Theorem~\ref{thm:progDI} with $\beta=1$, which implies that $f$ is an approximately optimal fair classifier for \ref{eq:progDI}.

\begin{proof} (of Theorem~\ref{thm:quantification})
	Our proof are divided into two parts.
	We first prove that $\rho_{q^\{\widehat{\Im}\}}(f^\star)\geq \tau$, i.e., $\min_{i\in [p]} q^{\widehat{\Im} }_i(f^\star)\geq \tau \cdot \min_{i\in [p]} q^{\widehat{\Im} }_i(f^\star)$.
	Then based on this claim, we show how to prove the theorem.
	W.l.o.g., we consider $q^\Im_1(f^\star)$ and $q^\Im_2(f^\star)$.
	By the definition of $f^\star$, we have 
	\begin{align}
	\label{eq:1}
	q^\Im_1(f)\geq \tau\cdot q^\Im_{2}(f)+ \kappa.
	\end{align}
	By the definition of $\kappa$ and the fact that $\tau\leq 1$, we have
	\begin{eqnarray}
	\label{eq:3}
	\begin{split}
	& q^\Im_1(f)-\tau\cdot q^\Im_{2}(f) \\
	\leq & q^{\widehat{\Im}}_1(f)+\max_{i\in [p], f\in \calF}\left|q^{\widehat{\Im}}_i(f)-q^\Im_i(f) \right|-\tau\cdot \left(q^{\widehat{\Im}}_{2}(f)-\max_{i\in [p], f\in \calF}\left|q^{\widehat{\Im}}_i(f)-q^\Im_i(f) \right|\right) \\
	= & q^{\widehat{\Im}}_1(f)-\tau \cdot q^{\widehat{\Im}}_{2}(f) + (1+\tau)\cdot \max_{i\in [p], f\in \calF}\left|q^{\widehat{\Im}}_i(f)-q^\Im_i(f) \right| \\
	\leq & q^{\widehat{\Im}}_1(f)-\tau \cdot q^{\widehat{\Im}}_{2}(f) +\kappa.
	\end{split}
	\end{eqnarray}
	Combining Inequalities~\eqref{eq:1} and~\eqref{eq:3}, we obtain the following
	\[
	q^{\widehat{\Im} }_1(f^\star) \geq \tau\cdot q^{\widehat{\Im} }_2(f^\star).
	\]
	Then by symmetry, we have $\min_{i\in [p]} q^{\widehat{\Im} }_i(f)\geq \tau \cdot \min_{i\in [p]} q^{\widehat{\Im} }_i(f)$.

	By this claim, we are ready to prove the theorem.
	By Theorem~\ref{thm:progDI}, the output $f$ satisfies the following:
	\begin{align} 
	\label{eq:4}
	\Pr_{\widehat{\Im} }\left[f\neq Y\right]\leq \Pr_{\widehat{\Im} }\left[ f^\star\neq Y\right],
	\end{align}
	and
	\begin{align}
	\label{eq:5}
	\min_{i\in [p]} q^{\widehat{\Im} }_i(f)\geq \tau \cdot \min_{i\in [p]} q^{\widehat{\Im} }_i(f)-\eps.
	\end{align}
	Similar to Inequality~\eqref{eq:3}, we have the following 
	\begin{align*}
	\min_{i\in [p]} q^{\Im}_i(f)- \tau \cdot \min_{i\in [p]} q^{\Im}_i(f) 
	\geq \min_{i\in [p]} q^{\widehat{\Im} }_i(f)- \tau \cdot \min_{i\in [p]} q^{\widehat{\Im} }_i(f) -\kappa.
	\end{align*}
	Combining the above inequality with Inequality~\eqref{eq:5}, we have
	\[
	\min_{i\in [p]} q^\Im_i(f)\geq \tau \cdot \min_{i\in [p]} q^\Im_i(f)-\eps- \kappa.
	\]
	It remains to prove that $\Pr_\Im\left[f\neq Y\right]\leq \Pr_\Im\left[ f^\star\neq Y\right]+2\cdot d_{TV}(\widehat{\Im},\Im)$.
	By Assumption (3), we have
	\begin{eqnarray*}
		\begin{split}
			\Pr_\Im\left[f\neq Y\right] =& \sum_{(x,z,y): f=y } \Pr_\Im\left[X=x,Z=z,Y=y\right]& \\
			\in &\sum_{(x,z,y): f=y } \Pr_{\widehat{\Im} }\left[X=x,Z=z,Y=y\right]\pm d_{TV}(\widehat{\Im},\Im)& (\text{Assumption (3)}) \\
			=& \Pr_{\widehat{\Im} }\left[f\neq Y\right] \pm d_{TV}(\widehat{\Im},\Im).
		\end{split}
	\end{eqnarray*}
	Similarly, we have $\Pr_\Im\left[f^\star\neq Y\right]\in \Pr_{\widehat{\Im} }\left[f^\star\neq Y\right]+ d_{TV}(\widehat{\Im},\Im)$.
	Combining with Inequality~\eqref{eq:4}, we complete the proof.
\end{proof}

\begin{remark}
	\label{remark:delta}
	For the fairness metric $\MD$ (Remark~\ref{remark:fair}), we can also design an algorithm similar to Algorithm~\ref{alg:meta_mult}.
	We only need to modify Line 2 by $L :=\lceil \frac{1+\tau}{\eps} \rceil$ (recall $\tau\in [-1,0]$), $a_i := (i-1)\cdot \eps$ and $b_i := i \cdot \eps-\tau$.
	The quantification of the output c$f$ is similar to Theorem~\ref{thm:quantification}.
	The main differences are 
	\[
	f^\star := \arg\min_{f\in \calF: \MD(f)\geq \tau+\kappa}\Pr_\Im\left[f\neq Y\right],
	\] 
	and the output $f$ satisfies that 
	\begin{enumerate}
		\item $\Pr_\Im\left[f\neq Y\right]\leq \Pr_\Im\left[ f^\star\neq Y\right]+2\cdot d_{TV}(\Im, \widehat{\Im})$; 
		\item $\MD(f)\geq \tau-\eps-\kappa$.
	\end{enumerate}
	The details are discussed in Appendix~\ref{app:delta}.
\end{remark}

\section{Details of Section~\ref{sec:algorithm}: Algorithm for \ref{eq:progfair}}
\label{app:algorithm}

In this section, we fulfill the details of Section~\ref{sec:algorithm} by proposing an algorithm that computes an optimal fair classifier for \ref{eq:progfair}.

In the following, we first prove Theorem~\ref{thm:not_linear} in the case that $\calF=\left\{0,1\right\}^\calX$ and $q\in \calQ_{\mathrm{lin} }$, by giving the characterization result (Theorem~\ref{thm:attribute}) and the computation result  (Theorem~\ref{thm:computation}, Lemma~\ref{lm:not_sensitive_find_lambda}). 
Then we propose Algorithm~\ref{alg:plugin} for \ref{eq:progfair}.
In Sections~\ref{sec:algorithm_provide} and~\ref{sec:algorithm_fractional}, we also discuss how to extend Algorithm~\ref{alg:plugin} to $\calF=\left\{0,1\right\}^{\calX\times [p]}$ and $q\in \calQ_{\mathrm{linf} }$.

\vspace{-.1in}
\subsection{Characterization Result in Theorem~\ref{thm:not_linear}}
\label{sec:unconstrained}

We first show the characterization of an optimal solution in Theorem~\ref{thm:not_linear}.
The proof idea is to reduce \ref{eq:progfair} to an unconstrained optimization problem by Lagrangian principle.
By Definition~\ref{def:frac_lin}, we assume 
\[
q_i(f)=\alpha^{(i)}_0+ \sum_{j=1}^{k} \alpha^{(i)}_j\cdot \Pr\left[f=1\mid G_i,\event^{(i)}_j\right]
\]
for any $f\in \calF$ and $i\in [p]$.
For simplicity, we denote $\pi := \Pr\left[Y=1\right]$ to be the underlying positive probability, and $\eta(x) := \Pr\left[Y=1\mid X=x\right]$ to be the positive probability conditioned on $X=x$.
For any $x\in \calX$, $i\in [p]$ and $j\in [k]$, we also denote 
\[
\pi^{(i)}_j := \Pr\left[G_i,\event^{(i)}_j \right], \quad \eta^{(i)}_j(x) := \Pr\left[G_i,\event^{(i)}_j\mid X=x\right].
\]
Then we can rewrite~\eqref{eq:s_lambda} by the following:
\begin{align}
\label{eq:s_lambda_cost}
s_\lambda(x):= \eta(x)-0.5+\sum_{i\in [p]}\lambda_i\cdot \left(\sum_{j=1}^{k} \frac{\alpha^{(i)}_j}{\pi^{(i)}_j}\cdot \eta^{(i)}_j(x)\right).
\end{align}
The following theorem shows that an optimal fair classifier is an instance-dependent threshold function based on $s_\lambda$.

\begin{theorem}[\textbf{Solution characterization}]
	\label{thm:attribute}
	%
	Given any parameters $\ell_i,u_i\in [0,1]$ $(i\in [p])$, there exists $\lambda^\star\in \R^p$ such that $\I[s_{\lambda^\star}(x)>0]$ is an optimal fair classifier for \ref{eq:progfair}.
\end{theorem}

\noindent
For analysis, we introduce randomized classifiers $f:\calX\rightarrow [0,1]$, in which $f$ predicts 1 with probability $f(x)$ for any $x\in \calX$.
Observe that this is a natural generalization of deterministic classifiers.
For preparation, we first rewrite the objective function $\Pr\left[f\neq Y\right]$ and the term $q_i(f)$ as a linear function of $f(X)$.

\begin{lemma}[Lemma 9 of \cite{menon2018the}]
	\label{lm:costrisk}
	For any classifier $f\in [0,1]^\calX$, 
	\[
	\Pr\left[f\neq Y\right]=\pi+\Exp_X\left[\left(1-2\eta(X)\right)\cdot f(X)\right].
	\]
\end{lemma}

\noindent
In the following lemma, we rewrite the term $q_i(f)$. 
For a randomized classifier $f: \calX\rightarrow [0,1]$ and any $i\in [p]$, we define 
\[
q_i(f):= \alpha^{(i)}_0+ \sum_{j=1}^{k} \alpha^{(i)}_j\cdot \Exp_X\left[f(X)\mid G_i,\event^{(i)}_j\right],
\]
which is a natural generalization of Definition~\ref{def:frac_lin} for $q\in \calQ_{\mathrm{lin} }$.

\begin{lemma}
	\label{lm:form}
	For any $f\in [0,1]^\calX$ and $i\in [p]$, 
	\[
	q_i(f) = \alpha^{(i)}_0+ \sum_{j=1}^{k}\frac{\alpha^{(i)}_j}{\pi^{(i)}_j}\cdot \Exp_X\left[\eta^{(i)}_j(X)\cdot f(X)\right].
	\]
\end{lemma}

\begin{proof}
	We have the following equality:
	\begin{align*}
	\begin{split}
	q_i(f) &  = \alpha^{(i)}_0+ \sum_{j=1}^{k} \alpha^{(i)}_j\cdot \Exp_X\left[f(X)\mid G_i,\event^{(i)}_j\right] \\
	& =  \alpha^{(i)}_0+ \sum_{j=1}^{k} \alpha^{(i)}_j\cdot \int_{x\in \calX} \Pr\left[X=x \mid G_i,\event^{(i)}_j\right] \cdot f(x) \\
	& =  \alpha^{(i)}_0+ \sum_{j=1}^{k} \alpha^{(i)}_j\cdot \int_{x\in \calX} \frac{\Pr\left[X=x\right]\cdot \Pr\left[G_i,\event^{(i)}_j\mid X=x\right]}{\Pr\left[G_i,\event^{(i)}_j\right]}\cdot f(x) \\
	& =  \alpha^{(i)}_0+ \sum_{j=1}^{k} \frac{\alpha^{(i)}_j}{\pi^{(i)}_j}\cdot \int_{x\in \calX} \Pr\left[X=x\right]\cdot \eta^{(i)}_j(x)\cdot f(x) \\
	& =   \alpha^{(i)}_0+ \sum_{j=1}^{k} \frac{\alpha^{(i)}_j}{\pi^{(i)}_j}\cdot \Exp_X\left[\eta^{(i)}_j(X)\cdot f(X)\right].
	\end{split}
	\end{align*}
\end{proof}

\noindent
Now we are ready to prove Theorem~\ref{thm:attribute}.
The main idea is to apply Lagrangian principle.

\begin{proof} (of Theorem~\ref{thm:attribute})	
	Let $f^\star$ denote an optimal solution of \ref{eq:progfair}.
	Denote $K:=\left\{f\in [0,1]^\calX: \ell_i\leq q_i(f)\leq u_i, \forall i\in [p]\right\}$.
	Let $g^\star$ denote an optimal solution of the following:
	\begin{align} 
	\label{eq:progfair_relaxcs}
	&\min_{g\in K} \Pr\left[g\neq Y\right]
	\end{align}
	By the definition of $K$, we know that $\Pr\left[g^\star\neq Y\right]\leq \Pr\left[f^\star\neq Y\right]$.
	By Lemmas~\ref{lm:costrisk} and~\ref{lm:form}, Program \eqref{eq:progfair_relaxcs} is a linear programming of $g(X)$.
	Then by strong duality for linear programs, \footnote{This implicitly assumes feasibility of the primal problem, i.e., the convex set $K$ is nonempty.} we have
	\begin{align}
	\label{program}
	\min_{g\in K} \Pr\left[g\neq Y\right] = & \max_{\nu, \zeta\in \R_{\geq 0}^p} \min_{g\in [0,1]^\calX} \Pr\left[g\neq Y\right] + \sum_{i\in [p]} \nu_i\cdot\left(q_i(g)-u_i\right)  -\sum_{i\in [p]} \zeta_i\cdot\left(q_i(g)-\ell_i\right).
	\end{align}
	Fix $\nu$ and $\zeta$.
	To solve the inner optimizer, it is equivalent to solve the following program:
	\begin{eqnarray*}
		\begin{split}
			& \min_{g\in [0,1]^\calX} \Pr\left[g\neq Y\right] + \sum_{i\in [p]} \nu_i\cdot \left( q_i(g)-\alpha^{(i)}_0 \right) - \sum_{i\in [p]} \zeta_i\cdot \left(q_i(g)-\alpha^{(i)}_0\right) & \\
			= & \min_{g\in [0,1]^\calX} \Pr\left[g\neq Y\right] + \sum_{i\in [p]} (\nu_i-\zeta_i)\cdot \left( \sum_{j=1}^{k}\frac{\alpha^{(i)}_j}{\pi^{(i)}_j}\cdot \Exp_X\left[\eta^{(i)}_j(X)\cdot g(X)\right]\right) & (\text{Lemma~\ref{lm:form} }) \\
			= & \min_{g\in [0,1]^\calX} \Pr\left[g\neq Y\right]-2\sum_{i\in [p]} \lambda_i\cdot \left( \sum_{j=1}^{k}\frac{\alpha^{(i)}_j}{\pi^{(i)}_j}\cdot \Exp_X\left[\eta^{(i)}_j(X)\cdot g(X)\right]\right) & (\text{letting }\lambda_i=\frac{\zeta_i-\nu_i}{2}).
		\end{split}
	\end{eqnarray*}
	Hence, there exists an optimal Lagrangian parameter $\lambda^\star\in \R^p$ such that 
	\begin{align*}
	g^\star = \arg\min_{g\in [0,1]^\calX} \Pr\left[g\neq Y\right]-2\sum_{i\in [p]} \lambda^\star_i\cdot \left( \sum_{j=1}^{k}\frac{\alpha^{(i)}_j}{\pi^{(i)}_j}\cdot \Exp_X\left[\eta^{(i)}_j(X)\cdot g(X)\right]\right).
	\end{align*}
	By Lemma~\ref{lm:costrisk}, we have
	\begin{eqnarray}
	\label{eq:opt}
	\begin{split}
	& \Pr\left[g\neq Y\right]-2\sum_{i\in [p]} \lambda^\star_i\cdot \left( \sum_{j=1}^{k}\frac{\alpha^{(i)}_j}{\pi^{(i)}_j}\cdot \Exp_X\left[\eta^{(i)}_j(X)\cdot g(X)\right]\right) & \\
	= & \pi + \Exp_X\left[\left(1-2\eta(X)\right)\cdot g(X)\right] &\\ 
	&-2\sum_{i\in [p]} \lambda^\star_i\cdot \left( \sum_{j=1}^{k}\frac{\alpha^{(i)}_j}{\pi^{(i)}_j}\cdot \Exp_X\left[\eta^{(i)}_j(X)\cdot g(X)\right]\right) & (\text{Lemma~\ref{lm:costrisk}})\\
	= & \pi + \Exp_X\left[\left(1-2\eta(X)-2\sum_{i\in [p]} \lambda^\star_i\cdot\left( \sum_{j=1}^{k}\frac{\alpha^{(i)}_j}{\pi^{(i)}_j}\right) \right)\cdot g(X)\right] &\\
	= & \pi -2\cdot \Exp_X\left[s_{\lambda^\star}(X)\cdot g(X)\right]. & (\text{Eq.~\eqref{eq:s_lambda_cost}})
	\end{split}
	\end{eqnarray}
	Therefore, $g^\star(x)=\I[s^\star(x)>0]$ is an optimal solution of Program~\eqref{eq:progfair_relaxcs}. 
	Recall that $\Pr\left[g^\star\neq Y\right]\leq \Pr\left[f^\star\neq Y\right]$ and $\I[s^\star(x)>0]$ is a deterministic classifier.
	Thus, $\I[s^\star(X)>0]$ is also an optimal fair classifier for \ref{eq:progfair}. 
	This completes the proof.
\end{proof}

\subsection{Computation Result in Theorem~\ref{thm:not_linear}}
\label{sec:computation}

We then discuss how to efficiently compute an optimal solution in Theorem~\ref{thm:not_linear}.
By Theorems~\ref{thm:attribute}, it remains to show how to compute the optimal Lagrangian parameters $\lambda^\star$.
The main idea is applying the explicit formulation of $g^\star=\I[s^\star(X)>0]$ to Program \eqref{program}.
Then computing the optimal Lagrangian parameters is equivalent to solving an unconstrained optimization problem.
We have the following theorem.

\begin{theorem}[\textbf{Solution computation}]
	\label{thm:computation}
	Given any parameters $\ell_i,u_i\in [0,1]$ $(i\in [p])$, there exists a polynomial-time algorithm that computes the optimal Lagrangian parameter $\lambda^\star$ as stated in Theorem~\ref{thm:attribute}.
\end{theorem}

\noindent
Theorem~\ref{thm:computation} is a corollary of the following lemma.

\begin{lemma}
	\label{lm:not_sensitive_find_lambda}
	In Theorem~\ref{thm:attribute}, the optimal Lagrangian parameter $\lambda^\star$ is the solution of the following optimization program:
	\begin{align}\tag{OPT-Lambda}
	\label{eq:lambda}
	\lambda^\star = \arg\min_{\lambda\in \R^p} \Exp_X\left[ (s_\lambda(X))_+ \right] + \sum_{i\in [p]} \left(\alpha^{(i)}_0-u_i\right)\lambda_i +\sum_{i\in [p]} \left(u_i-\ell_i\right)\cdot (\lambda_i)_+.
	\end{align}
\end{lemma}

\begin{proof}	
	Denote $K:=\left\{f\in [0,1]^\calX: \ell_i\leq q_i(f)\leq u_i, \forall i\in [p]\right\}$.
	By the proof of Theorem~\ref{thm:attribute}, we have
	\begin{align*}
	&\min_{f\in K} \Pr\left[f\neq Y\right] \\
	= & \max_{\nu, \zeta\in \R_{\geq 0}^p} \min_{f\in [0,1]^\calX} \Pr\left[f\neq Y\right] + \sum_{i\in [p]} \nu_i\cdot\left(q_i(f)-u_i\right) +\sum_{i\in [p]} \zeta_i\cdot\left(-q_i(f)+\ell_i\right) \\
	= & \max_{\nu, \zeta\in \R_{\geq 0}^p} \min_{f\in [0,1]^\calX} \Pr\left[f\neq Y\right]- \sum_{i\in [p]} (\zeta_i-\nu_i)\cdot q_i(f) +\sum_{i\in [p]} \left(\zeta_i \ell_i-\nu_i u_i\right).
	\end{align*}
	Let $\lambda_i=\frac{\zeta_i-\nu_i}{2}$ for all $i\in [p]$. 
	Considering the outer optimization,	we discuss the following cases.
	\begin{enumerate}
		\item If $\lambda_i\geq 0$, then $\zeta_i\geq \nu_i$.
		We have 
		\[
		\zeta_i \ell_i-\nu_i u_i = (2\lambda_i+\nu_i) \ell_i - \nu_i u_i =  2\lambda_i \ell_i +\nu_i(\ell_i-u_i)\stackrel{\ell_i\leq u_i}{\leq} 2\lambda_i \ell_i.
		\]
		The equality only holds if $\nu_i =0$.
		\item If $\lambda_i< 0$, then $\zeta_i< \nu_i$.
		We have 
		\[
		\zeta_i \ell_i-\nu_i u_i = \zeta_i \ell_i - (\zeta_i-2\lambda_i) u_i = 2\lambda_i u_i +\zeta_i(\ell_i-u_i)\stackrel{\ell_i\leq u_i}{\leq} 2\lambda_i u_i.
		\]
		The equality only holds if $\zeta_i =0$.
	\end{enumerate}
	By the above argument, we have
	\begin{align*}
	&\min_{f\in K} \Pr\left[f\neq Y\right] & \\
	= & \max_{\nu, \zeta\in \R_{\geq 0}^p} \min_{f\in [0,1]^\calX} \Pr\left[f\neq Y\right]- \sum_{i\in [p]} (\zeta_i-\nu_i)\cdot q_i(f) +\sum_{i\in [p]} \left(\zeta_i \ell_i-\nu_i u_i\right) & \\
	= & \max_{\lambda\in \R^p} \min_{f\in [0,1]^\calX} \Pr\left[f\neq Y\right]-2\sum_{i\in [p]} \lambda_i\cdot \left(\alpha^{(i)}_0+ \sum_{j=1}^{k}\frac{\alpha^{(i)}_j}{\pi^{(i)}_j}\cdot \Exp_X\left[\eta^{(i)}_j(X)\cdot f(X)\right]\right) &\\
	& +2\sum_{i\in [p]} \lambda_i\cdot \left( \I[\lambda_i\geq 0]\cdot \ell_i + \I[\lambda_i<0]\cdot u_i \right) & \\
	= & \max_{\lambda\in \R^p} \min_{f\in [0,1]^\calX}  \pi - 2\cdot \Exp_X\left[ s_\lambda(X)\cdot f(X) \right]-2\alpha^{(i)}_0\cdot \sum_{i\in [p]} \lambda_i &\\
	&+2\sum_{i\in [p]} \lambda_i\cdot \left( \I[\lambda_i\geq 0]\cdot \ell_i + \I[\lambda_i<0]\cdot u_i \right) & (\text{Eq.~\eqref{eq:opt} }) \\
	=& \max_{\lambda\in \R^p} \pi - 2\cdot \Exp_X\left[ (s_\lambda(X))_+ \right]-2\alpha^{(i)}_0\cdot \sum_{i\in [p]} \lambda_i +2\sum_{i\in [p]} \left(\ell_i-u_i\right)\cdot (\lambda_i)_+ + u_i \lambda_i & \\ &  (\text{letting $f(X)=\I[s^\star(X)>0]$}) & \\
	= & \max_{\lambda\in \R^p} \pi - 2\cdot \Exp_X\left[ (s_\lambda(X))_+ \right]+ 2\sum_{i\in [p]} \left(u_i-\alpha^{(i)}_0\right)\lambda_i +2\sum_{i\in [p]} \left(\ell_i-u_i\right)\cdot (\lambda_i)_+. &
	\end{align*}
	Hence, the optimal Lagrangian parameter $\lambda^\star$ is exactly the solution of \ref{eq:lambda}, which completes the proof.
\end{proof}

\paragraph{Solving \ref{eq:lambda}.}
To complete the proof of Theorem~\ref{thm:computation}, it remains to show how to compute $\lambda^\star$, i.e., solving the optimization problem~\ref{eq:lambda}. 
Define a function $\phi:\R^p\rightarrow \R$ by
\[
\phi(\lambda)=\Exp_X\left[ (s_\lambda(X))_+ \right] + \sum_{i\in [p]} \left(\alpha^{(i)}_0-u_i\right)\lambda_i +\sum_{i\in [p]} \left(u_i-\ell_i\right)\cdot (\lambda_i)_+.
\]
Note that the first term $\Exp_X\left[ (s_\lambda(X))_+ \right]$ can be rewritten as 
\[
\Exp_X\left[ (s_\lambda(X))_+ \right] = \int_{X\sim \Im} \left(\langle a(X), \lambda \rangle +b(X) \right)_+
\]
for some function $a: \calX\rightarrow \R_{\geq 0}$ and $b: \calX\rightarrow \R$.
Hence, $\Exp_X\left[ (s_\lambda(X))_+ \right]$ is a convex function of $\lambda$.
On the other hand, since $u_i-\ell_i\geq 0$, $\left(u_i-\ell_i\right)\cdot (\lambda_i)_+$ is also a convex function of $\lambda_i$.
Overall, $\phi(\lambda)$ is a convex function of $\lambda$.
Then, a natural idea is to apply standard convex optimization algorithms for solving \ref{eq:lambda}. 

If $\calX$ is finite, we can rewrite $\phi(\lambda)$ as a piecewise linear functions explicitly or apply subgradient descent.
However, $\calX$ can be infinite. 
In this case, we apply the stochastic subgradient method~\cite{boyd2008stochastic}.
We make a wild assumption that $\|\lambda^\star\|_2$ is bounded and $\min_{i\in [p],j\in [k]}\Pr\left[G_i,\event^{(i)}_j\right]$ is a constant away from 0. 

We first consider the subgradient of $\phi(\lambda)$.
Define $g:\R^p\rightarrow \R^p$ as follows: for any $\lambda\in \R^p$ and $i\in [p]$,
\begin{align*}
\label{eq:subgradient}
\begin{split}
g_i(\lambda):= \int_{x\in \calX} \I\left[s_{\lambda}(x) \geq 0\right]\cdot \sum_{j=1}^{k} \alpha^{(i)}_j\cdot \frac{\Pr\left[X=x\right]\cdot \eta^{(i)}_j(x)}{\pi^{(i)}_j} + \alpha^{(i)}_0-u_i + \left(u_i-\ell_i\right)\cdot \I\left[\lambda_i>0\right].
\end{split}
\end{align*}
It is not hard to check that $g(\lambda)$ is a subgradient of $\phi(\lambda)$.
However, since $\calX$ is infinite, we can not compute the subgradient $g(\lambda)$ directly.
Hence, we apply the stochastic subgradient method~\cite{boyd2008stochastic}. 
We first show how to construct an unbiased estimation of the subgradient $g(\lambda)$.
Given a $\lambda\in \R^p$, we draw a sample $(\dot{x},\dot{z},\dot{y})\sim \Im$.
\footnote{Here, we assume the existence of a sample oracle for $\Im$.}
Then we estimate $g(\lambda)$ by a stochastic vector $\tilde{g}(\lambda)$ where for each $i\in [p]$, 
\[
\tilde{g}_i(\lambda):= \I\left[s_{\lambda}(\dot{x}) \geq 0\right] \cdot \sum_{j=1}^{k} \frac{\alpha^{(i)}_j\cdot  \eta^{(i)}_j(\dot{x})}{\pi^{(i)}_j} + \alpha^{(i)}_0-u_i + \left(u_i-\ell_i\right)\cdot \I\left[\lambda_i>0\right].
\]
Note that $\Exp\left[\tilde{g}(\lambda)\right]=g(\lambda)$ which implies that $\tilde{g}$ is an unbiased estimation of $g$.
Our update rule is as follows:
\begin{enumerate}
	\item Initially, let $\lambda^{(0)}:=0$.
	\item Assume at iteration $t$, we have a point $\lambda^{(t)}\in \R^p$. 
	Let $\lambda^{(t+1)}:= \lambda^{(t)}-c_t \tilde{g}(\lambda^{(t)})$ where $c_t > 0$ is the $t$-th step size. 
\end{enumerate}

\noindent
Let $G:= \sup_{\lambda} \Exp\left[\|\tilde{g}(\lambda)\|_2^2\right]$ denote the supremum of the variance of $\tilde{g}(\lambda)$.
By~\cite[Section 3]{boyd2008stochastic}, we have
\[
\min_{t\in [T]} \Exp\left[\phi(\lambda^{(t)})\right] - \phi(\lambda^\star) \leq \frac{\|\lambda^\star\|_2^2+G\cdot \sum_{t\in [T]}c^2_t}{\sum_{t\in [T]}c_t}.
\]
By setting $T:=\frac{4G\cdot \|\lambda^\star\|_2^2}{\eps^2}$ and $c_t:=\sqrt{\frac{G\cdot \|\lambda^\star\|_2^2}{T}}$ ($\eps>0$), we have $\min_{t\in [T]} \Exp\left[\phi(\lambda^{(t)})\right]\leq \phi(\lambda^\star)+\eps$. 
Hence, our update rule of $\lambda$ converges if $G$ is bounded.
Next, we give an upper bound of $G$ by the following lemma.

\begin{lemma}
	\label{lm:subgradient}
	$G\leq 2\sum_{i\in [p]} \left(1+ |\alpha^{(i)}_0|\right)^2+\frac{2k^2p\cdot\left\|\alpha\right\|_{\infty}^2}{\min_{i\in [p],j\in [k]}\Pr\left[G_i,\event^{(i)}_j\right]}$
\end{lemma}

We defer the proof to Appendix~\ref{app:missing}. By the above lemma, since $\min_{i\in [p],j\in [k]}\Pr\left[G_i,\event^{(i)}_j\right]$ is a constant away from 0, $G$ is upper bounded and hence our update rule converges.
Combining Theorems~\ref{thm:attribute} and~\ref{thm:computation}, we directly obtain Theorem~\ref{thm:not_linear}.

\subsection{Algorithm for Theorem~\ref{thm:not_linear}.}
Now we are ready to propose an algorithm for Theorem~\ref{thm:not_linear}; see Algorithm~\ref{alg:plugin}.
The main idea is to compute the optimal Lagrangian parameter $\lambda^\star$ by Lemma~\ref{lm:not_sensitive_find_lambda} and then output the classifier by the formulation given in Theorem~\ref{thm:attribute}.
Note that our algorithm is similar to~\cite[Algorithm 1]{menon2018the}.
However, Menon and Williamson~\cite{menon2018the} did not show how to compute the optimal Lagrangian parameters.

\IncMargin{1em}
\begin{algorithm}
	\caption[The LOF caption]{\ref{eq:progfair}($\Im,q^{\Im}, \left\{\ell_i\right\},\left\{u_i\right\}$)}
	\label{alg:plugin}
	\SetKwInOut{Input}{Input}
	\SetKwInOut{Output}{Output}
	
	\Indm
	\Input{A distribution $\Im$ over $\calX\times [p] \times \left\{0,1\right\}$, a linear group performance function $q^\Im\in \calQ_{\mathrm{lin} } $ and fairness parameters $\ell_i,u_i\in [0,1]$ for each $i\in [p]$.}
	\Output{A classifier $f:\calX\rightarrow \left\{0,1\right\}$. } 
	\Indp
	\BlankLine
	For any $x\in \calX$, $i\in [p]$ and $j\in [k]$, $\eta(x) \leftarrow \Pr_\Im\left[Y=1\mid X=x\right]$ and $\eta^{(i)}_j(x) \leftarrow \Pr_\Im\left[G_i,\event^{(i)}_j\mid X=x\right]$. \\
	Compute $\lambda^\star$ to be the optimal solution of \ref{eq:lambda}. \\
	Compute $s_{\lambda^\star}: x\rightarrow \eta(x)-0.5+\sum_{i\in [p]}\lambda^\star_i\cdot \left(\sum_{j=1}^{k} \frac{\alpha^{(i)}_j}{\pi^{(i)}_j}\cdot \eta^{(i)}_j(x)\right)$. \\
	Return $f: x\rightarrow \I[s_{\lambda^\star}(x)>0]$.
\end{algorithm}
\DecMargin{1em}

\begin{remark}
	Observe that we do not need the full information of distribution $\Im$ in Algorithm~\ref{alg:plugin}.
	In fact, we only need the following information: $\Pr_{\Im}\left[X\right]$, $\pi^{(i)}_j :=\Pr_{\Im}\left[G_i,\event^{(i)}_j \right]$, $\eta(X) :=\Pr_{\Im}\left[Y=1\mid X\right]$ and $\eta^{(i)}_j(X) := \Pr_{\Im}\left[G_i,\event^{(i)}_j\mid X\right]$.
	This observation is useful when the underlying distribution $\Im$ is unknown, since we only need to estimate the above information instead of estimating the full distribution $\Im$.
\end{remark}

\subsection{Extension to $\calF=\left\{0,1\right\}^{\calX\times [p]}$} 
\label{sec:algorithm_provide}

In this section, we discuss how to extend Algorithm~\ref{alg:plugin} to the case that the sensitive attribute is used for prediction, i.e., $\calF=\left\{0,1\right\}^{\calX\times [p]}$.
We summarize the differences as follows.

\paragraph{Characterization.} 
Similarly, we denote $\pi := \Pr\left[Y=1\right]$ and $\eta(x,i) := \Pr\left[Y=1\mid X=x, G_i\right]$ for $i\in [p]$ and $x\in \calX$.
For any $x\in \calX$, $i,i'\in [p]$ and $j\in [k]$, we also denote 
\[
\pi^{(i)}_j := \Pr\left[G_i,\event^{(i)}_j \right], \quad \eta^{(i)}_j(x,i') := \Pr\left[G_i,\event^{(i)}_j\mid X=x,G_{i'}\right].
\]
By the definition of $G_i,\event^{(i)}_j$, we know that $\eta^{(i)}_j(x,i') =0$ if $i\neq i'$.
Pluging in $\widehat{x}:=(x,i)$ to Theorem~\ref{thm:attribute}, we directly have the following corollary.

\begin{corollary}[\textbf{Solution characterization for $\calF=\left\{0,1\right\}^{\calX\times [p]}$}]
	\label{cor:attribute_provide}
	Suppose $\calF=\left\{0,1\right\}^{\calX\times [p]}$.
	Given any parameters $\ell_i,u_i\in [0,1]$ $(i\in [p])$,
	there exists $\lambda^\star\in \R^p$ such that $\I[s_{\lambda^\star}(x,i)>0]$ is an optimal solution of \ref{eq:progfair}, where \[
	s_{\lambda^\star}(x,i)=\eta(x,i)-0.5+\lambda^\star_i\cdot \left(\sum_{j=1}^{k} \frac{\alpha^{(i)}_j}{\pi^{(i)}_j}\cdot \eta^{(i)}_j(x,i)\right).
	\]
\end{corollary}

\paragraph{Computation.} We still need to show how to compute the optimal Lagrangian parameters $\lambda^\star$.
Similar to Lemma~\ref{lm:not_sensitive_find_lambda}, we have the following lemma which shows that $\lambda^\star$ is the optimal solution of some convex program.

\begin{lemma}
	\label{lm:sensitive_find_lambda}
	In Corollary~\ref{cor:attribute_provide}, the optimal Lagrangian parameter $\lambda^\star$ is the solution of the following program:
	\begin{align*}
	\label{eq:lambda_provide}
	\lambda^\star = \arg\min_{\lambda\in \R^p} \Exp_{X,Z}\left[ \left(s_\lambda(X,Z)\right)_+ \right] + \sum_{i\in [p]} \left(\alpha^{(i)}_0-u_i\right)\lambda_i +\sum_{i\in [p]} \left(u_i-\ell_i\right)\cdot (\lambda_i)_+ ,
	\end{align*}
	where $s_\lambda(x,i)=\eta(x,i)-0.5+\lambda_i\cdot \left(\sum_{j=1}^{k} \frac{\alpha^{(i)}_j}{\pi^{(i)}_j}\cdot \eta^{(i)}_j(x,i)\right)$.
\end{lemma}

\begin{remark}
	Similarly, we do not need the full information of distribution $\Im$.
	In fact, we only need to have the following information for computing an optimal fair classifier: $\Pr_{\Im}\left[X\right]$, $\pi^{(i)}_j :=\Pr_{\Im}\left[G_i,\event^{(i)}_j \right]$, $\eta(X,Z) :=\Pr_{\Im}\left[Y=1\mid X,Z\right]$ and $\eta^{(i)}_j(X,Z) := \Pr_{\Im}\left[G_i,\event^{(i)}_j\mid X,Z\right]$.
\end{remark}

\subsection{Generalization to $q\in \calQ_{\mathrm{linf} }$}
\label{sec:algorithm_fractional}

In this section, we consider how to generalize Algorithm~\ref{alg:plugin} to $q\in \calQ_{\mathrm{linf} }$. %
By~\eqref{eq:fractional}, we assume for any $f\in \calF$ and $i\in [p]$,
\[
q_i(f) = \frac{ \alpha^{(i)}_0+ \sum_{j=1}^{k} \alpha^{(i)}_j\cdot \Pr\left[f=1\mid G_i,\event^{(i)}_j\right]}{\beta^{(i)}_0+ \sum_{j=1}^{l} \beta^{(i)}_j\cdot \Pr\left[f=1\mid G_i,\eventb^{(i)}_j\right]},
\]
where
\[
\alpha^{(i)}_0+ \sum_{j=1}^{k} \alpha^{(i)}_j\cdot \Pr_\Im\left[f=1\mid G_i,\event^{(i)}_j\right], \beta^{(i)}_0+ \sum_{j=1}^{l} \beta^{(i)}_j\cdot \Pr_\Im\left[f=1\mid G_i,\eventb^{(i)}_j\right]\geq 0.
\]
Then by simple calculation, the fairness constraint $\ell_i\leq q_i(f)\leq u_i$ is equivalent to the following constraints:
\begin{align*}
& \alpha^{(i)}_0-\ell_i \beta^{(i)}_0 + \sum_{j=1}^{k} \alpha^{(i)}_j\cdot \Pr\left[f=1\mid G_i,\event^{(i)}_j\right] -\ell_i \sum_{j=1}^{l} \beta^{(i)}_j\cdot \Pr\left[f=1\mid G_i,\eventb^{(i)}_j\right] \geq 0, ~\text{and} \\
& -\alpha^{(i)}_0  +u_i\beta^{(i)}_0- \sum_{j=1}^{k} \alpha^{(i)}_j\cdot \Pr\left[f=1\mid G_i,\event^{(i)}_j\right] +u_i \sum_{j=1}^{l} \beta^{(i)}_j\cdot \Pr\left[f=1\mid G_i,\eventb^{(i)}_j\right] \geq 0.
\end{align*}

\noindent
According to above constraints, we construct two linear group benefit functions $q^{(1)}$ and $q^{(2)}$: for any $f\in \calF$ and $i\in [p]$, denote
\begin{align*}
& q^{(1)}_i(f):=\alpha^{(i)}_0-\ell_i \beta^{(i)}_0 + \sum_{j=1}^{k} \alpha^{(i)}_j\cdot \Pr\left[f=1\mid G_i,\event^{(i)}_j\right] -\ell_i \sum_{j=1}^{l} \beta^{(i)}_j\cdot \Pr\left[f=1\mid G_i,\eventb^{(i)}_j\right], ~\text{and} \\
& q^{(2)}_i(f):= -\alpha^{(i)}_0  +u_i\beta^{(i)}_0- \sum_{j=1}^{k} \alpha^{(i)}_j\cdot \Pr\left[f=1\mid G_i,\event^{(i)}_j\right] +u_i \sum_{j=1}^{l} \beta^{(i)}_j\cdot \Pr\left[f=1\mid G_i,\eventb^{(i)}_j\right].
\end{align*}
Then, \ref{eq:progfair} w.r.t. to $q$ is equivalent to the following program:

\begin{align*} 
\begin{split}
\min_{f\in \calF} &~\Pr_{\Im}\left[f\neq Y\right] \\ 
s.t., & ~q^{(1)}_i(f) \geq 0, ~ \forall i\in [p] \\
& ~q^{(2)}_i(f) \geq 0, ~ \forall i\in [p].
\end{split}
\end{align*}

\noindent
The above program is similar to the case of linear group benefit functions.
The only difference is that for each $i\in [p]$, we have two linear constraints now.
However, it will only introduce double Lagrangian parameters to handle all linear constraints.
By Lagrangian principle, we can obtain Theorem~\ref{thm:fractional}.

We again define some notations for simplicity.
For any $x\in \calX$, denote $\eta(x):= \Pr\left[Y=1 \mid X=x\right]$.
For any $x\in \calX$, $i\in [p]$ and $j\in [k]$, denote 
\[
\eta^{(i)}_j(x):=\Pr\left[G_i,\event^{(i)}_j\mid X=x\right], \quad \pi^{(i)}_j:= \Pr\left[G_i,\event^{(i)}_j\right].
\]
For any $x\in \calX$, $i\in [p]$ and $j\in [l]$, denote  
\[
\xi^{(i)}_j(x):=\Pr\left[G_i,\eventb^{(i)}_j\mid X=x\right],\quad\omega^{(i)}_j:= \Pr\left[G_i,\eventb^{(i)}_j\right].
\]
For any $\nu,\zeta\in \R_{\geq 0}^p$, we define a function $s_{\nu,\zeta}: \calX\rightarrow \R$ by
\begin{align*}
s_{\nu,\zeta}(x)=&\eta(x)-0.5+\sum_{i\in [p]}\nu_i\cdot \left(\sum_{j\in [k]}\frac{\alpha^{(i)}_j}{\pi^{(i)}_j}\cdot \eta^{(i)}_j(x) - \ell_i\sum_{j\in [l]}\frac{ \beta^{(i)}_j}{\omega^{(i)}_j}\cdot \xi^{(i)}_j(x)\right) \\
&+ \zeta_i\cdot \left(-\sum_{j\in [k]}\frac{\alpha^{(i)}_j}{\pi^{(i)}_j}\cdot \eta^{(i)}_j(x) + u_i\sum_{j\in [l]}\frac{ \beta^{(i)}_j}{\omega^{(i)}_j}\cdot \xi^{(i)}_j(x)\right).
\end{align*}
Similar to~\eqref{eq:s_lambda}, $s_{\nu,\zeta}(x)$ can be regarded as the optimal threshold for the following fairness-aware classification problem:
\[
\min_{f\in \calF} \Pr\left[f\neq Y\right]-\sum_{i\in [p]} \nu_i\cdot q^{(1)}_i(f)-\sum_{i\in [p]} \zeta_i\cdot q^{(2)}_i(f).
\] 
Then we have the following theorem.

\begin{theorem}[\textbf{Solution characterization and computation for $q\in \calQ_{\mathrm{linf} }$}]
	\label{thm:fractional}
	Suppose $\calF=\left\{0,1\right\}^\calX$ and $q\in \calQ_{\mathrm{linf} }$.
	Given any parameters $\ell_i,u_i\in [0,1]$ $(i\in [p])$, there exists $\nu^\star, \zeta^\star\in \R_{\geq 0}^p$ such that
	$
	\I[s_{\nu^\star,\zeta^\star}(x)>0]
	$
	is an optimal fair classifier for \ref{eq:progfair}. 
	Moreover, we can compute the optimal Lagrangian parameters $\nu^\star$ and $\zeta^\star$ in polynomial time as a solution of the following convex program:
	\begin{align*}
	(\nu^\star, \zeta^\star) = \arg\min_{\nu,\zeta \in \R_{\geq 0}^p} \Exp_X\left[ (s_{\nu,\zeta}(X)_+ \right] + \sum_{i\in [p]} \nu_i \cdot \left(\alpha^{(i)}_0-\ell_i\beta^{(i)}_0\right)+\sum_{i\in [p]} \zeta_i\cdot \left(-\alpha^{(i)}_0+ u_i\beta^{(i)}_0\right).
	\end{align*}
\end{theorem}

\noindent
We omit the proof which is similar to Theorem~\ref{thm:not_linear}.
By the above theorem, there also exists a natural algorithm for $q\in \calQ_{\mathrm{linf} }$: firstly compute the optimal Lagrangian parameters $\nu^\star$ and $\zeta^\star$, then output an optimal fair classifier 	$
\I[s_{\nu^\star,\zeta^\star}(x)>0]
$.

\section{Empirical Results}
\label{sec:experiment}

\paragraph{Datasets.} 
We consider the following datasets in our experiments
\begin{itemize}
\item \textbf{Adult} income dataset~\cite{Dua:2017}, which records the demographics of 45222 individuals,
along with a binary label indicating whether the income of an individual is greater than 50k USD. 
We take gender to be the sensitive attribute, which is binary in the dataset.
\item \textbf{German} dataset \cite{Dua:2017}, records the attributes corresponding to around 1000 individuals with a label indicating positive or negative credit risk. 
Here again, we take gender to be the sensitive attribute, which is binary in the dataset.
\item \textbf{COMPAS} dataset \cite{compas}, compiled by Propublica, is a list of demographic data of criminal offenders along with a risk score. 
We refer the reader to \cite{larson2016we} for more details on how the data was analysed and compiled.
We take race to be the sensitive attribute, and for simplicity consider only those elements with race attribute either black or white.
\end{itemize}

\paragraph{Metrics.} 
Let $D$ denote the empirical distribution over the testing set. 
Given a group performance function $q$, we denote $\gamma_q$ to be the fairness metric $\rho_{q}$ under the empirical distribution $D$. 
For instance, given a classifier $f$, 
\[
\gamma_{\rm sr}(f) := \min_{i\in [p]}\Pr_{D} \left[f = 1 \mid Z = i\right]/\max_{i\in [p]}\Pr_{D} \left[f = 1 \mid Z = i\right].
\]

\paragraph{Algorithms and Benchmarks.} 

We consider three versions of Algorithm~\ref{alg:meta_mult}:
\begin{enumerate}
	\item Subject to $\tau_{\rm sr}$-fair (Algo~\ref{alg:meta_mult}-SR);
	\item Subject to $\tau_{\rm fdr}$-fair (Algo~\ref{alg:meta_mult}-FDR);
	\item Subject to $\tau_{\rm sr}$-fair and $\tau_{\rm fdr}$-fair (Algo~\ref{alg:meta_mult}-SR+FDR),
\end{enumerate}
and use three algorithms from the literature for comparison: 
\begin{enumerate}
	\item \textbf{COV} developed in~\cite{zafar2017fairness} with the goal of controlling the ratio $\gamma_{\rm sr}$; 
	\item \textbf{SHIFT} developed in \cite{hardt2016equality} designed to constrain the false positive parity and false negative parity; 
	\item \textbf{FPR-COV} and \textbf{FNR-COV} presented in~\cite{zafar2017fair}, designed to control the ratios $\gamma_{\rm fpr}$ and $\gamma_{\rm fnr}$ respectively.
\end{enumerate}
%

\paragraph{Experimental Setup.} 

We perform five repetitions, in which we divide the dataset uniformly at random into training (70$\%$) and testing (30$\%$) sets and report the average statistics of the above algorithms.
In  Algorithm~\ref{alg:meta_mult}, we set the error parameter $\eps$ to  $0.01$, and fit the estimated distribution $\widehat{\Im}$ in Line 1 using Gaussian Naive Bayes using SciPy~\cite{scipy}. 
For each dataset we run Algo~\ref{alg:meta_mult}-FDR for $\tau \in \{0.1, 0.2, \ldots, 1.0\}$, and plot the resulting $\gamma_{\rm sr}$ and accuracy. 

\subsection{Empirical Results}

Fig.~\ref{fig:sr} summarizes the tradeoff between mean value of accuracy and observed fairness $\gamma_{\rm sr}$. 
The points represents the mean value of $\gamma_{\rm sr}$ and accuracy, with the error bars representing the standard deviation.
We observe that Algo~\ref{alg:meta_mult}-SR can achieve higher $\gamma_{\rm sr}$ than than other methods. 
However, this gain in fairness comes at a loss; accuracy is decreasing in $\gamma_{\rm sr}$ for Algo~\ref{alg:meta_mult}-SR (albeit always above $75\%$). 
Even for lower values of $\gamma_{\rm sr}$, the accuracy is worse than that of \textbf{COV} and \textbf{SHIFT}; as suggested by Theorem~\ref{thm:quantification} , this is likely due to the fact that we use a very simplistic model for the empirical distribution $\Im$ -- we expect this to improve if we were to tune the fit. 
Similarly, Fig.~\ref{fig:fdr} summarizes the tradeoff between accuracy and $\gamma_{\rm fdr}$.
Here we observe that Algo~\ref{alg:meta_mult}-FDR achieves both high accuracy and high $\gamma_{\rm fdr}$, as does \textbf{FPR-COV}, while the other methods are worse with respect to fairness and/or accuracy.
We think that all the algorithms perform well with respect to output fairness $\gamma_{\rm fdr}$ is likely because the unconstrained optimal classifier for Algorithm~\ref{alg:meta_mult} achieves $\gamma_{\rm fdr}=0.84$ (see Table~\ref{tab:SR_FDR}), i.e., the \textbf{Adult} dataset is already nearly unbiased for gender with respect to FDR.
%
Empirically, we find that the observed fairness is almost always close to the target constraint. 
The output fairness and accuracy of the classifier against the input measure $\tau$ is depicted in Fig.~\ref{fig:sr_tau} and Fig.~\ref{fig:fdr_tau}. 
We plot all points from all the training/test splits in these figures.\footnote{Note that the trade-offs in these figures appear non-monotone because they represent the average results for all five training-test splits of the dataset. Within each partition, they are monotone.}

We also examine the performance of our methods and the baselines with respect to other fairness metrics $\gamma_q$ and report their mean and standard deviation. 
For Algo~\ref{alg:meta_mult}-SR, Algo~\ref{alg:meta_mult}-SR+FDR and COV, we consider only classifiers corresponding to $\gamma_{\rm sr} \geq 0.8$, while for Algo~\ref{alg:meta_mult}-FDR, FPR-COV, FNR-COV and SHIFT, we choose the classifier corresponding to $\gamma_{\rm fdr} \geq 0.8$.
Different methods are better at optimizing different fairness metrics -- the key difference is that Algo~\ref{alg:meta_mult} can optimize different metrics depending on the given parameters, whereas other methods do not have this flexibility; e.g., here we constrain fairness with respect to SR  and FDR (for which the maximal values of $\gamma_{\rm sr}$ and $\gamma_{\rm fdr}$ are attained), but we could instead constrain with respect to any other $q$ if desired.
Interestingly, although Algo~\ref{alg:meta_mult}-SR and Algo~\ref{alg:meta_mult}-FDR do not achieve the highest accuracy overall, both have significantly higher accuracy parity than other methods ($\gamma_{\rm ar}\approx 0.9$).
Furthermore, we can consider multiple fairness constraints simultaneously; Algo~\ref{alg:meta_mult}-SR+FDR can achieve both $\gamma_{\rm sr} > 0.7$ and $\gamma_{\rm fdr} > 0.7$, while remaining methods can not ($\gamma_{\rm sr}<0.45$ or $\gamma_{\rm fdr}<0.55$).
Unfortunately, this does come at a loss of accuracy, likely due to the difficulty of simultaneously achieving accuracy and multiple fairness metrics~\cite{chouldechova2017fair,kleinberg2017inherent}.

Empirical analysis of other datasets (COMPAS \cite{compas} and German dataset \cite{Dua:2017}) are presented in Appendix~\ref{sec:other_experiments}. The primary observations with respect to these datasets are presented below.

The performance of Algo~\ref{alg:meta_mult}-FDR with respect to other algorithms on German dataset is depicted in Figures \ref{fig:german_fdr} and \ref{fig:german_fdr_tau} . 
From Fig~\ref{fig:german_fdr_tau}, we observe that the classifier is able to satisfy the input fairness constraint every time, i.e., for all values of input $\tau_{\rm fdr}$, the observed fairness of the classifier, $\gamma_{\rm fdr}$, is greater than or almost equal to $\tau_{\rm fdr}$.
Furthermore, as shown in Fig~\ref{fig:german_fdr}, the maximum  $\gamma_{\rm fdr}$ value achieved by Algo~\ref{alg:meta_mult}-FDR is around 0.99, while amongst other algorithms, the maximum achieved is around 0.85. 
Similarly for Algo~\ref{alg:meta_mult}-SR, whose results are presented in Figures \ref{fig:german_sr} and \ref{fig:german_sr_tau}, we see that for almost all values of input $\tau_{\rm sr}$, we satisfy the input fairness constraint (except when $\tau_{\rm sr}$ is almost 1, in which case observed $\gamma_{\rm sr}$ is close to 0.98).

Figures \ref{fig:compass_fdr} and \ref{fig:compass_fdr_tau} show how the Algo~\ref{alg:meta_mult}-FDR fares with respect to other algorithms on COMPAS dataset. 
In general the output classifier is not able to achieve very high output fairness. 
Algo~\ref{alg:meta_mult}-FDR achieves a maximum $\gamma_{\rm fdr}$ of around 0.80, while \textbf{SHIFT} is able to $\gamma_{\rm fdr}$ value as high as 0.98.
We believe that this is because the empirical distribution considered for the algorithm (multivariate Gaussian) is not a good fit for the data given, and correspondingly as predicted by Theorem~\ref{thm:quantification}, we incur a loss in the output fairness. 
A similar explanation can be considered for the performance of Algo~\ref{alg:meta_mult}-SR, presented in Figures \ref{fig:compass_sr} and \ref{fig:compass_sr_tau}.

\setlength\tabcolsep{1em}
\begin{table}[t]
	\centering
	\caption{Performance (mean and std.) of different methods with respect to accuracy and fairness metrics $\gamma_q$ in Table~\ref{tab:result}. 
		We also offer the performance of an unconstrained optimal classifier for Algorithm~\ref{alg:meta_mult} for comparison. \vspace{0.3cm}
	} 
	
	\scriptsize{
		\begin{tabular}{|m{1.6cm}<{\centering}|m{0.40cm}<{\centering}|m{0.40cm}<{\centering}|m{0.40cm}<{\centering}|m{0.40cm}<{\centering}|m{0.40cm}<{\centering}|m{0.40cm}<{\centering}|m{0.40cm}<{\centering}|m{0.40cm}<{\centering}|m{0.40cm}<{\centering}|m{0.40cm}<{\centering}|m{0.40cm}<{\centering}|m{0.40cm}<{\centering}|} \hline
			\multicolumn{12}{|c|}{This paper}  \tabularnewline \hline
			& Acc. & $\gamma_{\rm sr}$ & $\gamma_{\rm fpr}$ & $\gamma_{\rm fnr}$ & $\gamma_{\rm tpr}$ & $\gamma_{\rm tnr}$ & $\gamma_{\rm ar}$ & $\gamma_{\rm fdr}$ & $\gamma_{\rm for}$ & $\gamma_{\rm ppr}$ & $\gamma_{\rm npr}$  \tabularnewline \hline
			Unconstrained  & 0.83 (0.00) & 0.33 (0.03) & 0.30 (0.02) & \textbf{0.87} (0.05) & 0.86 (0.06) & 0.94 (0.00) & 0.86 (0.01) & \textbf{0.84} (0.07) & 0.34 (0.03) & \textbf{0.93} (0.03) & 0.87 (0.01) \tabularnewline \hline
			Algo~\ref{alg:meta_mult}-SR & 0.77 (0.01) & \textbf{0.89} (0.05) & 0.51 (0.04) & 0.55 (0.10) & 0.81 (0.03) & 0.82 (0.02) & \textbf{0.90} (0.02) & 0.46 (0.03) & 0.21 (0.04) & 0.39 (0.04) & 0.88 (0.00)\tabularnewline \hline
			Algo~\ref{alg:meta_mult}-FDR  & 0.83 (0.00) & 0.32 (0.04) & 0.27 (0.05) & 0.78 (0.07) & 0.86 (0.06) & 0.88 (0.01) & \textbf{0.89} (0.05) & \textbf{0.85} (0.03) & 0.36 (0.03) & \textbf{0.93} (0.04) & 0.89 (0.00)  \tabularnewline \hline
			Algo~\ref{alg:meta_mult}-SR+FDR & 0.44 (0.13) & \textbf{0.84} (0.04) & \textbf{0.83} (0.09) & 0.21 (0.27) & \textbf{0.96} (0.01) & 0.36 (0.37) & 0.48 (0.26) & 0.70 (0.04) & 0.15 (0.16) & 0.34 (0.06) & \textbf{0.95} (0.03)\tabularnewline  \hline
		\end{tabular}
		
		\vspace{0.2cm}			
		
		\begin{tabular}{|m{1.6cm}<{\centering}|m{0.40cm}<{\centering}|m{0.40cm}<{\centering}|m{0.40cm}<{\centering}|m{0.40cm}<{\centering}|m{0.40cm}<{\centering}|m{0.40cm}<{\centering}|m{0.40cm}<{\centering}|m{0.40cm}<{\centering}|m{0.40cm}<{\centering}|m{0.40cm}<{\centering}|m{0.40cm}<{\centering}|m{0.40cm}<{\centering}|} 		\multicolumn{12}{|c|}{Baselines}  \tabularnewline \hline
			& Acc. & $\gamma_{\rm sr}$ & $\gamma_{\rm fpr}$ & $\gamma_{\rm fnr}$ & $\gamma_{\rm tpr}$ & $\gamma_{\rm tnr}$ & $\gamma_{\rm ar}$ & $\gamma_{\rm fdr}$ & $\gamma_{\rm for}$ & $\gamma_{\rm ppr}$ & $\gamma_{\rm npr}$  \tabularnewline \hline
			\textbf{COV}~\cite{zafar2017fairness}  & 0.79 (0.28) & 0.83 (0.01) & 0.63 (0.06) & 0.27 (0.19) & 0.76 (0.07) & 0.79 (0.10) & 0.81 (0.06) & 0.55 (0.12) & 0.10 (0.05) & 0.44 (0.11) & 0.86 (0.02) \tabularnewline \hline
			\textbf{FPR-COV}~\cite{zafar2017fair} & \textbf{0.85} (0.01) & 0.41 (0.07)&0.39 (0.08) & \textbf{0.87} (0.10) & \textbf{0.91} (0.07) & 0.94 (0.01) &\textbf{0.88} (0.01) & 0.80 (0.08) &0.29 (0.05) & 0.91 (0.04)& 0.87 (0.02) \tabularnewline \hline
			\textbf{FNR-COV}~\cite{zafar2017fair} & \textbf{0.85} (0.01)& 0.22 (0.05) & 0.14 (0.04)& 0.61 (0.09) & 0.67 (0.10) & 0.89 (0.01) & \textbf{0.88} (0.04) & 0.80 (0.05) & \textbf{0.50} (0.05) & 0.92 (0.02) & 0.91 (0.01)\tabularnewline \hline
			\textbf{SHIFT}~\cite{hardt2016equality}  & 0.81 (0.01) & 0.50 (0.11) & 0.40 (0.16) & \textbf{0.90} (0.06) & 0.84 (0.09) & \textbf{0.98} (0.00) & 0.83 (0.01) & \textbf{0.84} (0.06) & 0.31 (0.02) & \textbf{0.96} (0.02) & 0.82 (0.01)  \tabularnewline \hline
		\end{tabular}
	}
	
	\label{tab:SR_FDR}
	\vspace{0.15in}
\end{table}

\begin{figure}
		\centering
		\includegraphics[width=1\linewidth]{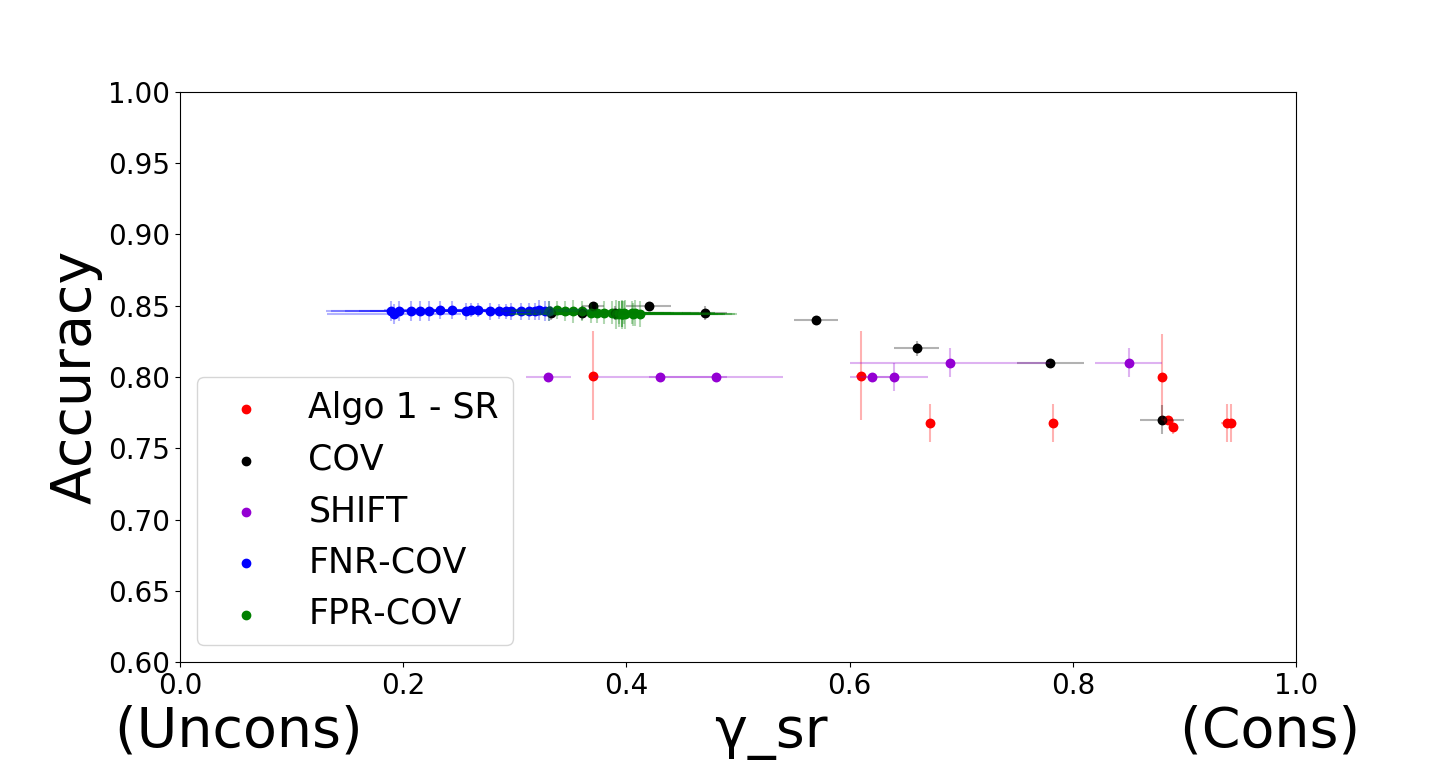}
		\caption{Acc. vs. $\gamma_{\rm sr}$. Algo~\ref{alg:meta_mult}-SR can achieve better fairness with respect to SR than any other method, albeit at a loss to accuracy.}
		\label{fig:sr}
\end{figure}
	\begin{figure}
			\includegraphics[width=1\linewidth]{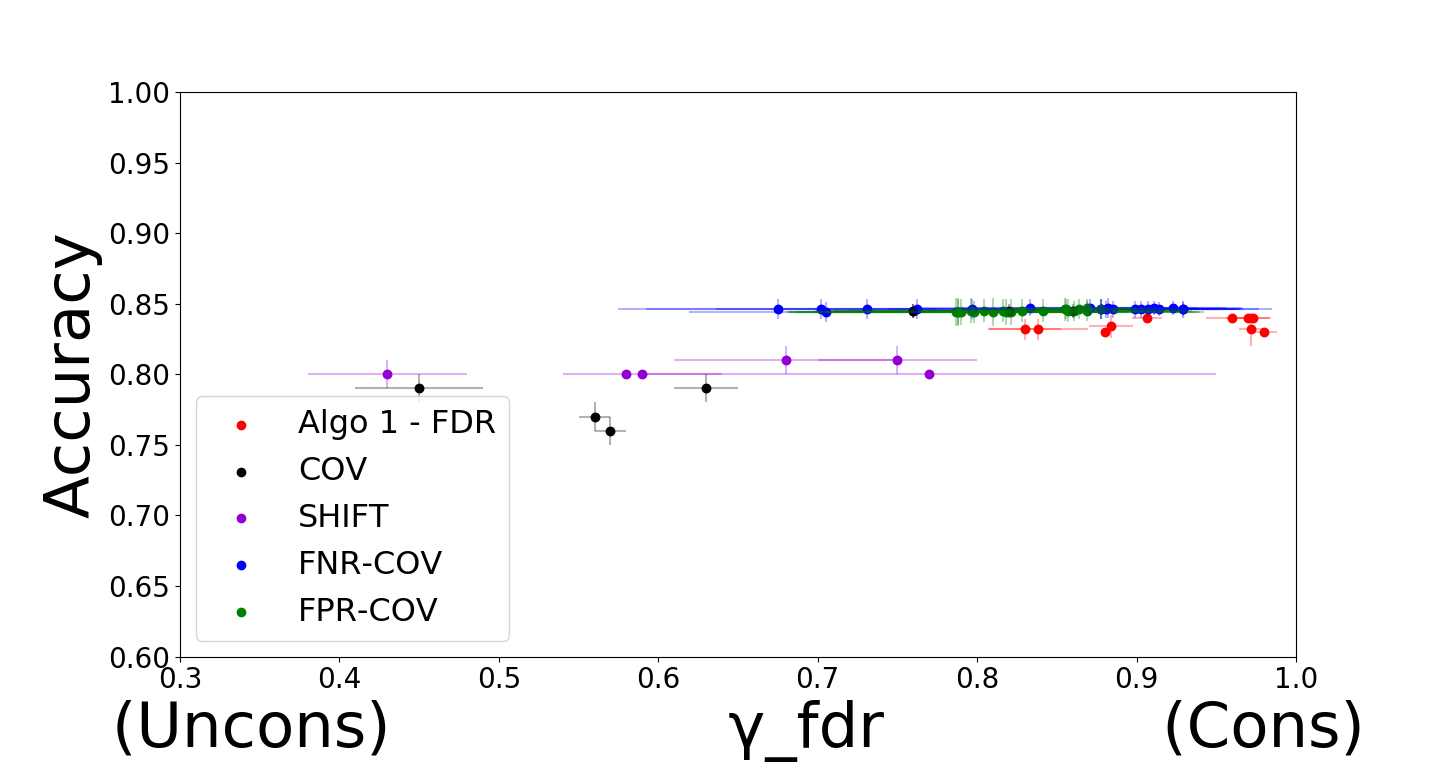}
		\caption{Acc. vs. $\gamma_{\rm fdr}$. Algo~\ref{alg:meta_mult}-FDR Algo~\ref{alg:meta_mult}-FDR achieves better fairness with respect to FDR and is indistinguishable with respect to accuracy.}
		\label{fig:fdr}
\end{figure}

\begin{figure}
		\centering
		\includegraphics[width=1\linewidth]{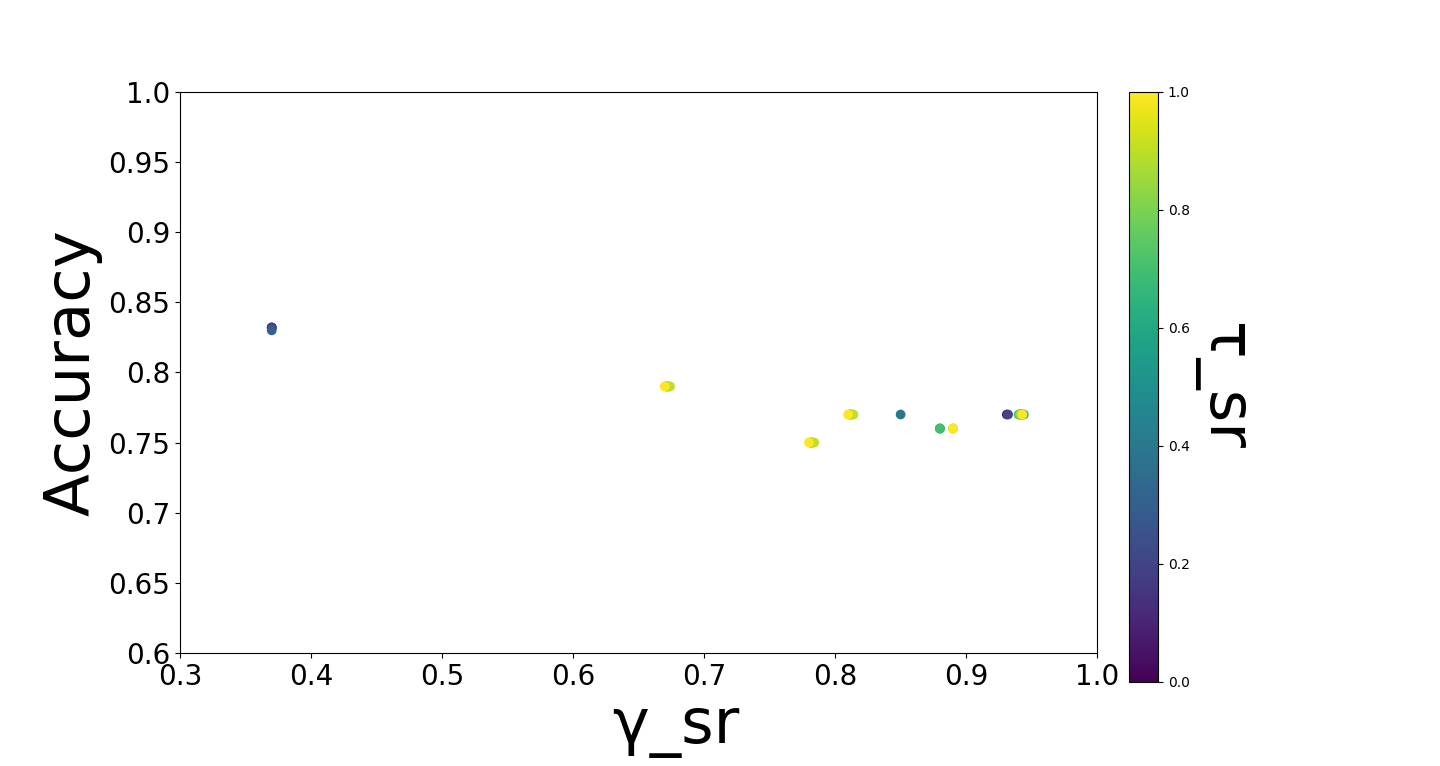}
		\caption{Acc. vs. $\gamma_{\rm sr}$. Algo~\ref{alg:meta_mult}-SR for different values of input $\tau_{\rm sr}$.}
		\label{fig:sr_tau}
		\end{figure}

	\begin{figure}
		\centering
			\includegraphics[width=1\linewidth]{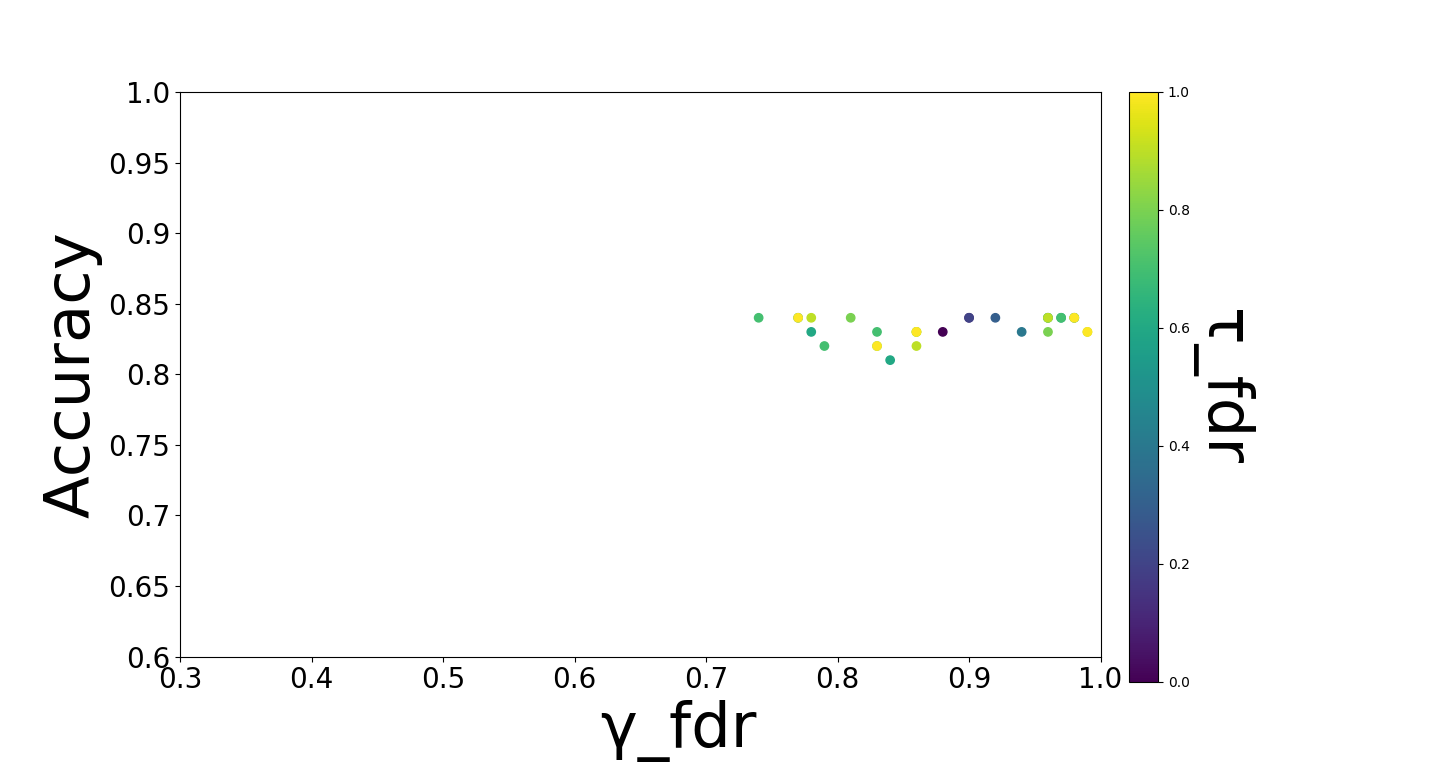}
		\caption{Acc. vs. $\gamma_{\rm fdr}$. Algo~\ref{alg:meta_mult}-FDR for different values of input $\tau_{\rm fdr}$.}
		\label{fig:fdr_tau}

\end{figure}

\section{Conclusion and Discussion}
\label{sec:conclusion}

We propose a framework for fair classification that can handle many existing fairness definitions in the literature. 
In particular, to the best of our knowledge, our framework is the first that can ensure predictive parity and has provable guarantees, which addresses an open problem proposed in~\cite{zafar2017fairness}.

This paper opens several possible directions for future work. 
Firstly, it would be important to evaluate this algorithm with other datasets in order to better evaluate the tradeoff between fairness and accuracy in a variety of real-world scenarios. 
We also believe it would be valuable to extend our framework to other commonly used loss functions (e.g., $l_2$-loss or AUC) and other classifiers (e.g., margin-based classifiers or score-based classifiers).
In this paper, we consider two fairness metrics $\DI$ and $\MD$.
Other examples like AUC and correlation (see the survey~\cite{zliobaite2017measuring}) could also be worth considering.
%

\bibliographystyle{plain}
\bibliography{references}

\appendix

\section{Proof of Lemma~\ref{lm:subgradient}}
\label{app:missing}

\begin{proof}
	By definition, we first rewrite $\tilde{g}(\lambda)$ as the sum of two vectors.
	Let $h\in \R^p$ denote a random vector where 
	\[
	h_i:=\I\left[s_{\lambda}(\dot{x})\geq 0\right]\cdot \sum_{j=1}^{k} \frac{\alpha^{(i)}_j\cdot  \eta^{(i)}_j(\dot{x})}{\pi^{(i)}_j}.
	\] 
	Also denote $h'\in \R^p$ to be 
	\[
	h'_i:= \alpha^{(i)}_0-u_i + \left(u_i-\ell_i\right)\cdot \I\left[\lambda_i>0\right].
	\]
	Note that
	\begin{eqnarray*}
		\begin{split}
			&\|\tilde{g}(\lambda)\|_2^2 & \\ 
			= & \|h+h'\|_2^2 & (\text{Defn. of $\tilde{g}(\lambda)$})\\
			\leq & \left(\|h'\|_2 + \left\|h\right\|_2\right)^2 & (\text{triangle ineq.}) \\
			\leq & 2\cdot \|h'\|_2^2+ 2\cdot\left\|h\right\|_2^2. &
		\end{split}
	\end{eqnarray*}
	Hence, we only need to bound the two terms $\|h'\|_2^2$ and $\left\|h\right\|_2^2$.
	We first bound $\|h'\|_2^2$. 
	For any $i\in [p]$,
	\begin{align*}
	|h'_i|= &\left|\alpha^{(i)}_0-u_i + \left(u_i-\ell_i\right)\cdot \I\left[\lambda_i>0\right] \right| \\
	\leq & \max\left\{\left|\alpha^{(i)}_0-u_i + u_i-\ell_i\right|,\left|\alpha^{(i)}_0-u_i  \right|\right\} \\
	\leq & u_i + |\alpha^{(i)}_0| \\
	\leq & 1+|\alpha^{(i)}_0|.
	\end{align*}
	Thus, we have $\left\|h'\right\|_2^2\leq \sum_{i\in [p]} \left(1+ |\alpha^{(i)}_0|\right)^2$ which is always bounded.
	On the other hand, we bound $ \Exp\left[\|h\|_2^2\right]$.
	By definition, we have
	\begin{eqnarray*}
		\begin{split}
			\label{eq:variance}
			&  \Exp\left[\|h\|_2^2\right] \\
			\leq &\Exp\left[\sum_{i\in [p]}\left( \sum_{j=1}^{k} \frac{\alpha^{(i)}_j\cdot  \eta^{(i)}_j(\dot{x})}{\pi^{(i)}_j}\right)^2\right] & (\text{Defn. of $h$})\\
			= & \int_{x\in \calX} \Pr\left[X=x\right]\cdot \sum_{i\in [p]}\left( \sum_{j=1}^{k} \frac{\alpha^{(i)}_j\cdot  \eta^{(i)}_j(x)}{\pi^{(i)}_j}\right)^2 & \\
			\leq & k\cdot \int_{x\in \calX} \Pr\left[X=x\right]\cdot \sum_{i\in [p]}\sum_{j=1}^{k} \left( \frac{\alpha^{(i)}_j\cdot  \eta^{(i)}_j(x)}{\pi^{(i)}_j}\right)^2 & \\
			\leq & k\cdot\left\|\alpha\right\|_{\infty}^2\cdot \int_{x\in \calX} \Pr\left[X=x\right]\cdot \sum_{i\in [p]}\sum_{j=1}^{k} \left( \frac{\eta^{(i)}_j(x)}{\pi^{(i)}_j}\right)^2 & \\
			= & k\cdot\left\|\alpha\right\|_{\infty}^2\cdot \int_{x\in \calX} \sum_{i\in [p]}\sum_{j=1}^{k} \frac{\Pr\left[X=x,G_i,\event^{(i)}_j\right]\cdot \Pr\left[G_i,\event^{(i)}_j\mid X=x\right]}{\Pr^2\left[G_i,\event^{(i)}_j\right]} & (\text{Defn. of $\eta^{(i)}_j$ and $\pi^{(i)}_j$})\\
			\leq & k\cdot\left\|\alpha\right\|_{\infty}^2\cdot \int_{x\in \calX} \sum_{i\in [p]}\sum_{j=1}^{k} \frac{\Pr\left[X=x,G_i,\event^{(i)}_j\right]}{\Pr^2\left[G_i,\event^{(i)}_j\right]} & \\
			\leq & k\cdot\left\|\alpha\right\|_{\infty}^2\sum_{i\in [p]}\sum_{j=1}^{k} \frac{1}{\Pr\left[G_i,\event^{(i)}_j\right]} & \\
			\leq & \frac{k^2p\cdot\left\|\alpha\right\|_{\infty}^2}{\min_{i\in [p],j\in [k]}\Pr\left[G_i,\event^{(i)}_j\right]}.
		\end{split}
	\end{eqnarray*}
	This completes the proof.
\end{proof}

\section{Existing Group Performance Functions are Linear-Fractional}
\label{app:fair_notion}

In this section, we discuss existing group performance functions listed in Table~\ref{tab:result}.
We prove that they are all linear-fractional and many of them are even linear.
We have the following two lemmas.

\begin{lemma}
	\label{lm:linear}
	$q$ is a linear group performance function for statistical/conditional statistical/true positive/false positive/true negative/false negative/accuracy rate.
\end{lemma}

\begin{proof}
	For statistical/conditional statistical/true positive/false negative rate, $q$ is obviously a linear group performance function by definition.
	For false positive and true negative rates, since
	\[
	\Pr\left[f=0\mid Y=1, G_i\right] = 1-\Pr\left[f=1\mid Y=1, G_i\right] \quad \text{and},
	\]
	\[
	\Pr\left[f=0\mid Y=0, G_i\right] = 1-\Pr\left[f=1\mid Y=0, G_i\right],
	\]
	$q$ is also linear.
	So we only need to prove that the case of accuracy rate.

	Recall that for accuracy rate, $q_i(f) = \Pr\left[f=Y\mid G_i\right]$ for $i\in [p]$ and $f\in \calF$.
	Then, we have
	\begin{align}
	\label{eq:accuracy}
	\begin{split}
	& q_i(f) = \Pr\left[f=Y\mid G_i\right] \\
	=& \Pr\left[Y=1\mid G_i\right]\cdot \Pr\left[f=1\mid Y=1,G_i\right] + \Pr\left[Y=0\mid G_i\right]\cdot \Pr\left[f=0\mid Y=0,G_i\right] \\
	= & \Pr\left[Y=1\mid G_i\right]\cdot \Pr\left[f=1\mid Y=1,G_i\right] +\Pr\left[Y=0\mid G_i\right]\cdot \left(1-\Pr\left[f=1\mid Y=0,G_i\right]\right) \\
	= & \Pr\left[Y=0\mid G_i\right]+\Pr\left[Y=1\mid G_i\right]\cdot \Pr\left[f=1\mid Y=1,G_i\right]\\
	&-\Pr\left[Y=0\mid G_i\right]\cdot \Pr\left[f=1\mid Y=0,G_i\right].
	\end{split}
	\end{align}
	Let $k:=2$, $\alpha^{(i)}_0:= \Pr\left[Y=0\mid G_i\right]$, $\alpha^{(i)}_1:= \Pr\left[Y=1\mid G_i\right]$, $\alpha^{(i)}_2:= -\Pr\left[Y=0\mid G_i\right]$, $G_i,\event^{(i)}_1:= (Y=1)$ and $G_i,\event^{(i)}_2:= (Y=0)$.
	Pluging the above values into Equality~\eqref{eq:accuracy}, we have that
	\[
	q_i(f) = \alpha^{(i)}_0+ \sum_{j=1}^{k} \alpha^{(i)}_j\cdot \Pr\left[f=1\mid G_i,\event^{(i)}_j\right],
	\]
	which completes the proof.
\end{proof}

\noindent
The remaining group performance functions in Table~\ref{tab:result} are not linear, but are still linear-fractional.

\begin{lemma}
	\label{lm:fractional}
	$q$ is a linear-fractional group performance function for false discovery/false omission/positive predictive/negative predictive rate.
\end{lemma}

\begin{proof}
	For false discovery rate, recall that $q_i(f) = \Pr\left[Y=0\mid f=1, G_i\right]$ for any $f\in \calF$ and $i\in [p]$.
	Then, we have
	\begin{align}
	\label{eq:false_discovery}
	\begin{split}
	q_i(f) = & \Pr\left[Y=0\mid f=1,G_i\right] \\
	=& \frac{\Pr\left[Y=0, f=1,G_i\right]}{\Pr\left[ f=1,G_i\right]} \\
	= & \frac{\Pr\left[Y=0, G_i\right]\cdot \Pr\left[f=1\mid Y=0, G_i\right]}{\Pr\left[G_i\right]\cdot \Pr\left[ f=1\mid G_i\right]}.
	\end{split}
	\end{align}
	Let $\alpha^{(i)}_1:= \Pr\left[Y=0, G_i\right]$ and $\beta^{(i)}_1:= \Pr\left[G_i\right]$, $\event^{(i)}_1:=(Y=0)$ and $\eventb^{(i)}_1:=\emptyset$.
	Pluging the above values into Equality~\eqref{eq:false_discovery}, we have
	\[
	q_i(f) = \frac{ \alpha^{(i)}_1\cdot \Pr\left[f=1\mid G_i,\event^{(i)}_1\right]}{\beta^{(i)}_1\cdot \Pr\left[f=1\mid G_i,\eventb^{(i)}_1\right]},
	\]
	which is linear-fractional by \eqref{eq:fractional}.
	By a similar argument, we can also prove false omission/positive predictive/negative predictive rate is linear-fractional.
\end{proof}

\section{Algorithms for \ref{eq:progDI} with Multiple Sensitive Attributes and Multiple Group Benefit Functions}
\label{app:multi}

In practice, we can have multiple sensitive attributes, e.g., gender and ethnicity.
We may also want to compute a classifier satisfying multiple fairness metrics, e.g., statistical parity and equalized odds.
Suppose there are $m$ sensitive attributes $Z_1,\ldots,Z_m$ where each $Z_j\in \left\{1,2,\ldots,p_i\right\}$.
In this case, domain $\calD$ changes to $\calD:=\calX \times [p_1]\times \cdots \times [p_i]\times \left\{0,1\right\}$.
Let $\Im$ denote the joint distribution over $\calD$. 
Let $\calF$ denote the collection of all possible classifiers.
%
For $i\in [m]$ and $j\in [p_i]$, we define $G^{(i)}_j$ to be the set of all $(x,z_1,\ldots,z_m,y)\in \calD$ with $z_i=j$.
Note that all groups $G^{(i)}_j$ form an $m$-partitions of $\calD$.
For each sensitive attribute $Z_i$, we denote a linear-fractional group performance function $q^{(i)}:\calF\rightarrow [0,1]^{p_i}$.

Note that the above setting also encompasses the case of multiple group performance functions for a sensitive attribute, e.g., letting $Z_1$ and $Z_2$ represent the same sensitive attribute but correspond to different group performance functions $q^{(1)}$ and $q^{(2)}$.
Next, we generalize \ref{eq:progDI} by the above setting as follows.

\begin{definition}[\textbf{Multi-$\DI$-Fair}]
	\label{def:multi_DI}
	For any $f\in \calF$ and $i\in [m]$, define
	\begin{equation*}
	\label{eq:multi_DI}
	\rho_i(f) = \min_{j\in [p_i]} q^{(i)}_j(f)/ \max_{j\in [p_i]} q^{(i)}_j(f)  \in [0,1].
	\end{equation*}
	Define a fairness metric $\rho:\calF\rightarrow [0,1]^m$ as follows: for any $f\in \calF$, $\rho(f):=\left(\rho_1(f),\ldots,\rho_m(f)\right)$.
	Given $\tau\in [0,1]^m$, a classifier $f$ is said to satisfy $\tau$-rule if for any $i\in [m]$, $\rho_i(f)\geq \tau_i$. 
	We consider the following fair classification program introduced by $\tau$-rule:
	\begin{equation} \tag{Multi-$\DI$-Fair}
	\label{eq:multi_progDI}
	\begin{split}
	& \min_{f\in \calF} \Pr_\Im\left[f\neq Y\right] \\ 
	& s.t.,~ \rho_i(f)\geq \tau_i, \forall i\in [m].
	\end{split}
	\end{equation}
\end{definition}

\noindent
We also generalize \ref{eq:progfair} as follows.

\begin{definition}[\textbf{Classification with fairness constraints for multiple sensitive attributes and multiple group performance functions}]
	\label{def:multi}
	For each $i\in [m]$, given $\ell^{(i)}_j,u^{(i)}_j \geq 0$ for all $j\in [p_i]$, the fairness constraint for group $G^{(i)}_j$ is defined by
	\[
	\ell^{(i)}_j\leq q^{(i)}_j(f)\leq u^{(i)}_j.
	\]
	The classification problem with fairness constraints is defined by the following program:
	\begin{align} \tag{Multi-Group-Fair}
	\label{eq:multi_progfair}
	\begin{split}
	&\min_{f\in \calF} \Pr_\Im\left[f\neq Y\right] \\ 
	& s.t.,~ \ell^{(i)}_j\leq q^{(i)}_j(f)\leq u^{(i)}_j, ~ \forall i\in [m], j\in [p_i].
	\end{split}
	\end{align}
\end{definition}

\paragraph{Reduction from \ref{eq:multi_progDI} to \ref{eq:multi_progfair}.}
Assume that $m$ is constant.
Similar to Section~\ref{sec:reduction}, we show that if \ref{eq:multi_progfair} is polynomial-time solvable, then \ref{eq:multi_progDI} is also polynomial-time approximately solvable.
The difference is that we need to apply \ref{eq:multi_progfair} roughly $O(\eps^{-m})$ times.
We state the generalized theorems as follows.

\begin{theorem}[\textbf{Reduction from \ref{eq:multi_progDI} to \ref{eq:multi_progfair} }]
	\label{thm:multi_progDI}
	Given $\tau\in [0,1]^m$, let $f^\star_\tau $ denote an optimal fair classifier for \ref{eq:multi_progDI}. 
	Suppose there exists a $\beta$-approximate algorithm $A$ for \ref{eq:multi_progfair}.
	Then for any constant $0<\eps<1/2$, there exists an algorithm that computes a classifier $f\in \calF$ such that 
	\begin{enumerate}
		\item $\Pr_\Im\left[f\neq Y\right]\leq \beta \cdot \Pr_\Im\left[f^\star_\tau\neq Y\right]$;
		\item for any $i\in [m]$, $\min_{j\in [p_i]} q^{(i)}_j(f)\geq \tau_i \cdot \min_{j\in [p]} q^{(i)}_j(f)-\eps$.
	\end{enumerate}
	by applying $A$ at most $\prod_{i\in [m]} \lceil\frac{\tau_i}{\eps} \rceil$ times.
\end{theorem}

\begin{proof}
	For each $i\in [m]$, let $T_i:=\lceil \frac{\tau_i}{\eps} \rceil$. 
	For each $i\in [m]$ and $j\in [T_i]$, denote $a^{(i)}_j := (i-1)\cdot \eps$ and $b^{(i)}_j := i\cdot \eps/\tau_i$.

	For each tuple $\Delta=(\delta_1,\ldots,\delta_m)\in [T_1]\times \cdots \times [T_m]$, we construct a program $P_{\Delta}$ to be \ref{eq:multi_progfair} with $\ell^{(i)}_j=a^{(i)}_{\delta_i}$ and $u^{(i)}_j=b^{(i)}_{\delta_i}$ ($i\in [m]$ and $j\in [p_i]$).
	Then we apply $A$ to compute $f_\Delta\in \calF$ as a solution of $P_{\Delta}$.
	Note that we apply $A$ at most $T_1\times \cdots \times T_m = \prod_{i\in [m]} \lceil\frac{\tau_i}{\eps} \rceil$ times.
	Among all $f_\Delta$, we output $f$ such that $\Pr_\Im\left[f\neq Y\right]$ is minimized.
	By the same argument as for Theorem~\ref{thm:progDI}, we can verify that $f$ satisfies the conditions of the theorem.
	It completes the proof.
\end{proof}

\paragraph{Algorithm for \ref{eq:multi_progDI}.}

Similar to Appendix~\ref{app:algorithm}, we can propose a polynomial-time algorithm that computes an optimal fair classifier for \ref{eq:multi_progfair} by Lagrangian principle.
The only difference is that we need at most $2\cdot\sum_{i\in [m]} p_i$ many Lagrangian parameters $\nu^{(i)}_j$ and $\zeta^{(i)}_j$, where each $\nu^{(i)}_j$ corresponds to the constraint $\ell^{(i)}_j\leq q^{(i)}_j(f)$ and each $\zeta^{(i)}_j$ corresponds to the constraint $ q^{(i)}_j(f)\leq u^{(i)}_j$.
The output classifier is again an instance-dependent threshold function, but with more Lagrangian parameters.

Then combining with Theorems~\ref{thm:multi_progDI}, we have a meta-algorithm that computes an approximately optimal fair classifier for \ref{eq:multi_progDI}.
We summarize the result by the following corollary.

\begin{corollary}[\textbf{Algorithm for \ref{eq:multi_progDI} }]
	\label{cor:multi_algorithm}
	Suppose $m$ is constant.
	For any $\tau\in [0,1]^m$, let $f^\star_\tau$ denote an optimal fair classifier for \ref{eq:multi_progDI}.
	Then for any constant $\eps>0$, there exists a polynomial-time algorithm that computes an approximately optimal fair classifier $f\in \calF$ satisfying that
	\begin{enumerate}
		\item $\Pr_\Im\left[f\neq Y\right]\leq \Pr_\Im\left[ f^\star_\tau\neq Y\right]$; 
		\item For any $i\in [m]$, $\min_{j\in [p_i]} q^{(i)}_j(f)\geq \tau_i \cdot \min_{j\in [p_i]} q^{(i)}_j(f)-\eps$.
	\end{enumerate}
\end{corollary}

\begin{remark}
	Similar to Section~\ref{sec:relaxed_algorithm}, assume that we only have $N$ samples $\left\{(x_i,z_i,y_i)\right\}_{i\in [N]}$ drawn from $\Im$, instead of knowing $\Im$ directly. 
	We can propose an algorithm almost the same to Algorithm~\ref{alg:meta_mult}: first estimate $\Im$ by $\widehat{\Im}$ and then solve a family of \ref{eq:multi_progfair} based on $\widehat{\Im}$ by Theorems~\ref{thm:multi_progDI}.

	The quantification of our algorithm is exactly the same as Theorem~\ref{thm:quantification}, except that there are $m$ group performance functions $q^{(i)}$. 
\end{remark}

\begin{remark}
	Another fairness metric $\MD$ can also be generalized to multiple cases.
	For each $f\in \calF$ and $i\in [m]$, define $\delta_i(f):=\min_{j\in [p_i]}q^{(i)}_j(f)-\max_{j\in [p_i]}q^{(i)}_j(f)$.
	Given $\tau\in [-1,0]^m$, the goal is to compute a classifier $f$ that minimizes the prediction error and $\delta_i(f)\geq \tau_i$ for all $i\in [m]$.
	We call this problem \textsf{Multi-$\MD$-Fair}.
	For this fairness metric, we also have a corollary similar to Corollary~\ref{cor:multi_algorithm}.

	\begin{corollary}[\textbf{Algorithm for \textsf{Multi-$\MD$-Fair}}]
		\label{cor:multi_MD_algorithm}
		Suppose $m$ is constant.
		For any $\tau\in [0,1]^m$, let $f^\star_\tau$ denote an optimal fair classifier for \textsf{Multi-$\MD$-Fair}.
		Then for any constant $\eps>0$, there exists a polynomial-time algorithm that computes an approximately optimal fair classifier $f\in \calF$ satisfying that
		\begin{enumerate}
			\item $\Pr_\Im\left[f(X)\neq Y\right]\leq \Pr_\Im\left[ f^\star_\tau(X)\neq Y\right]$; 
			\item For any $i\in [m]$, $\delta_j(f)\geq \tau_i -\eps$.
		\end{enumerate}
	\end{corollary}
	
\end{remark}

\section{Another Algorithm for \ref{eq:progDI} with Linear Group Performance Functions}
\label{app:another_algorithm}

For the case that all $q^{(i)}\in \calQ_{\mathrm{lin} }$, \ref{eq:progDI} is itself a linear program of $f$.
Hence, we can also apply Lagrangian principle for \ref{eq:progDI}, instead of introducing fairness constraints.
This setting generalizes the fair classification problem with statistical parity or true positive parity considered in~\cite{menon2018the}, and moreover, encompasses all fairness metrics with linear group performance functions.
In this section, we propose another algorithm for this setting.
Our main theorem is as follows.

\begin{theorem}	[\textbf{Algorithm for \ref{eq:progDI} with linear group performance functions}]
	\label{thm:algorithm_linear}
	Assume that $m$ is a constant and $q^{(i)}\in \calQ_{\mathrm{lin} }$ for any $i\in [m]$. 
	There exists a polynomial-time algorithm for \ref{eq:progDI}.
\end{theorem}

\noindent
For simplicity, we again consider the case that $\calF=\left\{0,1\right\}^{\calX}$ and $m=1$.
By Definition~\ref{def:frac_lin}, assume that $q_i(f)=\alpha^{(i)}_0+ \sum_{j=1}^{k} \alpha^{(i)}_j\cdot \Pr\left[f=1\mid G_i,\event^{(i)}_j\right]$ for $f\in \calF$ and $i\in [p]$.
Similar to Appendix~\ref{sec:algorithm_provide}, this setting can be extended to $\calF=\left\{0,1\right\}^{\calX\times [p]}$.
On the other hand, by Appendix~\ref{app:multi}, this setting can also be generalized to multiple sensitive attributes and multiple group performance functions ($m>1$).

The main idea is still to apply Lagrangian principle.
Note that $\DI(f)\geq \tau$ is equivalent to the following $p(p-1)$ constraints: for any $i\neq j\in [p]$,
\begin{align}
\label{eq:DI_constraint}
\begin{split}
& q_i(f)=\alpha^{(i)}_0+ \sum_{l=1}^{k} \alpha^{(i)}_l\cdot \Pr\left[f=1\mid G_i,\event^{(i)}_l\right] \\
\geq & \tau\cdot q_j(f) = \tau \cdot \alpha^{(i)}_0+ \sum_{l=1}^{k} \alpha^{(i)}_l\cdot \Pr\left[f=1\mid G_i,\event^{(i)}_l\right].
\end{split}
\end{align}
By rearranging, the above inequality is expressible as a linear constraint $a^\top f+b\leq 0$.
Then \ref{eq:progDI} is a linear program of $f$ and an optimal fair classifier should be still an instance-dependent threshold function.
We introduce a Lagrangian parameter $\lambda_{ij}$ for each constraint~\eqref{eq:DI_constraint}. 
Recall that for any $x\in \calX$, $i\in [p]$ and $j\in [k]$, we denote $\eta(x) := \Pr\left[Y=1\mid X=x\right]$ and $\eta^{(i)}_j(x) := \Pr\left[G_i,\event^{(i)}_j\mid X=x\right]$.
Similarly, we define a threshold function $s_\lambda:\calX\rightarrow \R$ for any $\lambda\in \R^{(p-1)p}$ as follows:
\[
s_\lambda(x):= \eta(x)-0.5+\sum_{i\neq j\in [p]} \lambda_{ij}\cdot \left(\sum_{l\in [k]}\frac{\alpha^{(i)}_l}{\pi^{(i)}_l}\cdot \eta^{(i)}_l(x) -\tau\cdot\sum_{l\in [k]}\frac{\alpha^{(j)}_l}{\pi^{(j)}_l}\cdot \eta^{(j)}_l(x) \right).
\]

\begin{theorem}[\textbf{Characterization and computation for \ref{eq:progDI} with linear group performance functions}]
	\label{thm:char_linear}
	Suppose $m$ is a constant and each $q^{(i)}\in \calQ_{\mathrm{lin} }$.
	Given any parameters $\tau_1,\ldots,\tau_m\in [0,1]$, there exists $\lambda^\star\in \R_{\geq 0}^{(p-1)p}$ such that
	$
	\I[s_{\lambda^\star}(x)>0]
	$
	is an optimal fair classifier for \ref{eq:progDI}. 
	Moreover, we can compute the optimal Lagrangian parameters $\lambda^\star$ in polynomial time as a solution of the following convex program:
	\begin{align*}
	\lambda^\star = \arg\min_{\lambda \in \R_{\geq 0}^{(p-1)p}} \Exp_X\left[ (s_{\lambda}(X))_+ \right] + \sum_{i\neq j\in [p]} \lambda_{ij} \cdot \left(\alpha^{(i)}_0-\tau\cdot \alpha^{(j)}_0\right).
	\end{align*}
\end{theorem}

\noindent
The proof is similar to Theorem~\ref{thm:attribute} and Lemma~\ref{lm:not_sensitive_find_lambda}, and hence we omit it.
Then Theorem~\ref{thm:algorithm_linear} is a direct conclusion of Theorem~\ref{thm:char_linear}.

We also consider the setting that only $N$ samples $\left\{(x_i,z_i,y_i)\right\}_{i\in [N]}$ drawn from $\Im$ are given instead of knowing $\Im$. 
Different from Algorithm~\ref{alg:meta_mult}, we only need to estimate $\Im$ by $\widehat{\Im}$ and apply the algorithm stated in Theorem~\ref{thm:char_linear}.
We call this algorithm \textsf{Meta2}.
The quantification of \textsf{Meta2} is almost the same to Theorem~\ref{thm:quantification} except that the error parameter $\eps$ should not exist, i.e., the output classifier $f$ of \textsf{Meta2} satisfies $\min_{i\in [p]}q_i(f)\geq \tau\cdot \max_{i\in [p]}q_i(f)-\kappa$. 

\begin{remark}[\textbf{Comparison between Algorithms~\ref{alg:meta_mult} and~\textsf{Meta2}}]
	Both Algorithm~\ref{alg:meta_mult} and~\textsf{Meta2} can handle multiple linear group performance functions.
	We analyze their advantages as follows:
	\begin{itemize}
		\item (Advantages of Algorithm~\ref{alg:meta_mult}.) Algorithm~\ref{alg:meta_mult} can also handle linear-fractional group performance functions but \textsf{Meta2} can not.
		On the other hand, \textsf{Meta2} introduces $(p-1)p$ Lagrangian parameters, while Algorithm~\ref{alg:meta_mult} only introduces $p$ Lagrangian parameters by applying Algorithm~\ref{alg:plugin}.
		When the number $p$ is large, introducing more Lagrangian parameters will increase the running time a lot and make the formulation of the output classifier more complicated.
		\item (Advantages of \textsf{Meta2}.) \textsf{Meta2} does not introduce an additional error parameter $\eps$ as Algorithm~\ref{alg:plugin}.
		Moreover, \textsf{Meta2} only needs to solve a convex optimization problem, while Algorithm~\ref{alg:plugin} solves around $\eps^{-1}$ optimization problems.
		If we require the error parameter $\eps$ to be extremely small, the running time of Algorithm~\ref{alg:plugin} can be much longer than \textsf{Meta2}.
	\end{itemize}
\end{remark}

\section{Details of Remark~\ref{remark:delta} for Fairness Metric $\MD$}
\label{app:delta}

In this section, we discuss another fairness metric $\MD$.
Assume there is only one sensitive attribute $Z\in \left\{1,2,\ldots,p\right\}$.
Also assume that $\calF=\left\{0,1\right\}^\calX$ and $q^\Im\in \calQ_{\mathrm{lin} }$ is a linear-fractional group performance function.
Given $\tau\in [-1,0]$, we consider the following fair classification problem induced by $\MD$, which captures existing constrained optimization problems~\cite{zafar2017fair,zafar2017fairness,menon2018the} as special cases.
\begin{equation} \tag{$\delta$-Fair}
\label{eq:progdelta}
\begin{split}
& \min_{f\in \calF} \Pr_\Im\left[f\neq Y\right] \\ 
& s.t.,~ \delta_{q^\Im}(f)=\min_{i\in [p]}q^\Im_i(f)- \max_{i\in [p]}q^\Im_i(f)\geq \tau.
\end{split}
\end{equation}

\noindent
Similar to Theorem~\ref{thm:progDI}, we first reduce \ref{eq:progdelta} to \ref{eq:progfair} by the following theorem.
Then combining with Algorithm~\ref{alg:plugin} (designed for \ref{eq:progfair}), we prove the existence of an efficient algorithm that computes an approximately optimal classifier for \ref{eq:progdelta}.

\begin{theorem}[\textbf{Reduction from \ref{eq:progdelta} to \ref{eq:progfair}}]
	\label{thm:progdelta}
	Given $\tau\in [-1,0]$, let $f^\star_\tau$ denote an optimal fair classifier for \ref{eq:progdelta}.
	%
	Given a $\beta$-approximate algorithm $A$ for \ref{eq:progfair} ($\beta\geq 1$) and any $\eps>0$, there exists an algorithm that applies $A$ at most $\lceil (1+\tau)/\eps \rceil$ times and computes a classifier $f\in \calF$ such that 1) $\Pr\left[f\neq Y\right]\leq \beta\cdot\Pr\left[f^\star_\tau\neq Y\right]$; 2) $\MD(f)\geq \tau-\eps$.
\end{theorem}

\begin{proof}
	Let $T:=\lceil (1+\tau)/\eps \rceil $. 
	For each $i\in [T]$, denote $a_i := (i-1)\cdot \eps$ and $b_i := i \cdot \eps-\tau$.
	For each $i\in [T]$, we construct a \ref{eq:progfair} program $P_i$  with $\ell_j=a_i$ and $u_j=b_i$ for all $j\in [p]$.
	Then we apply $A$ to compute $f_i\in \calF$ as a solution of $P_i$.
	Among all $f_i$, we output $f$ such that $\Pr\left[f\neq Y\right]$ is minimized.
	Next, we verify that $f$ satisfies the conditions in the theorem.

	Note that $a_i\geq b_i+\tau-\eps$ for each $i\in [T]$. 
	We have $\MD(f)=\min_{i\in [p]}q_i(f)-\max_{i\in [p]}q_i(f)\geq \tau-\eps$.
	On the other hand, assume that $$(j-1)\cdot \eps\leq \min_{i\in [p]}q_i(f^\star_\tau) < j\cdot \eps$$ for some $j\in [T]$.
	Since $f^\star_\tau$ is a feasible solution of \ref{eq:progDI},  
	$\max_{i\in [p]}q_i(f^\star_\tau)\leq \min_{i\in [p]}q_i(f^\star_\tau)-\tau < j\cdot \eps-\tau$.
	Hence, $f^\star_\tau$ is a feasible solution of Program $P_j$.
	By the definitions of $A$ and $f$, we have 
	$\Pr\left[f\neq Y\right]\leq \Pr\left[f_j\neq Y\right]\leq \beta\cdot \Pr\left[f^\star_\tau\neq Y\right]$.
\end{proof}

\noindent
By Theorem~\ref{thm:algorithm_linear}, there exists a 1-approximate algorithm for \ref{eq:progfair}.
Thus, there exists an efficient algorithm to approximately solve \ref{eq:progdelta}.
Similar to Appendix~\ref{app:multi}, this conclusion can also be generalized to multiple sensitive attributes and multiple group performance functions.

We also consider the practical setting that only $N$ samples $\left\{(x_i,z_i,y_i)\right\}_{i\in [N]}$ drawn from $\Im$ are given, instead of $\Im$.
Similar to Algorithm~\ref{alg:meta_mult}, we propose the following algorithm \textsf{Meta-$\MD$}: first estimate $\Im$ by $\widehat{\Im}$ on samples $\left\{(x_i,z_i,y_i)\right\}_{i\in [N]}$, and then compute a classifier based on $\widehat{\Im}$ by solving a family of Programs \ref{eq:progfair} as stated in Theorem~\ref{thm:progdelta}.
The only difference from Algorithm~\ref{alg:meta_mult} is that Line 2 of Algorithm~\ref{alg:meta_mult} should be changed to $L :=\lceil \frac{1+\tau}{\eps} \rceil$, $a_i := (i-1)\cdot \eps$ and $b_i := i \cdot \eps-\tau$.
Similar to Theorem~\ref{thm:quantification}, we have the following quantification result for \textsf{Meta-$\MD$}.

\begin{theorem}
	\label{thm:quantification_delta}
	Let $\kappa:=2\cdot \max_{i\in [p], f\in \calF}\left|q^{\widehat{\Im}}_i(f)-q^\Im_i(f) \right|$.
	Define an approximately optimal fair classifier $f^\star$ for \ref{eq:progdelta} by
	\begin{align*}
	\begin{split}
	f^\star:= & \arg\min_{f\in \calF} \Pr_\Im\left[f\neq Y\right] \\
	& s.t.,\delta_{q^\Im}(f)\geq \tau+\kappa.
	\end{split}
	\end{align*}
	Then Algorithm \textsf{Meta-$\MD$} outputs a classifier $f$ such that 
	\begin{enumerate}
		\item $\Pr_\Im\left[f\neq Y\right]\leq \Pr_\Im\left[ f^\star\neq Y\right]+2\cdot d_{TV}(\widehat{\Im},\Im) $,
		\item $\delta_{q^\Im}(f)\geq \tau-\eps-\kappa$.
	\end{enumerate}
\end{theorem}

\begin{proof}
	The proof is almost the same to Theorem~\ref{thm:quantification}.
	We first prove that $\delta_{q^{\widehat{\Im}}}(f^\star)\geq \tau$.
	The key is to show the following inequality: for any $i,j\in [p]$,
	\[
	q^{\widehat{\Im}}_i(f^\star) - q^{\widehat{\Im}}_j(f^\star) \geq q^{\Im}_i(f^\star) - q^{\Im}_j(f^\star) - \kappa \geq \delta_{q^{\Im}}(f^\star)-\kappa \geq \tau.
	\]
	Then we apply Theorem~\ref{thm:progdelta} to relate the quantification of $f$ with $f^\star$ under $\widehat{\Im}$.
	More concretely, we prove that 

	1) $\Pr_{\widehat{\Im}}\left[f\neq Y\right]\leq \Pr_{\widehat{\Im}}\left[ f^\star\neq Y\right]$; 
	
	2) $\delta_{q^{\widehat{\Im}}}(f)\geq \tau-\eps$.

	\noindent
	Finally, we transform the distribution $\widehat{\Im}$ back to $\Im$, which introduces the additional errors $2\cdot d_{TV}(\widehat{\Im},\Im)$ and $\kappa$ as stated in the theorem.
	It completes the proof.
\end{proof}

\section{Price of Fairness Constraints}
\label{sec:price}

An important concern of fair classification is the tradeoff between accuracy and fairness. 
We would like to quantify the tradeoff by measure of the underlying distribution $\Im$.
This quantification may also guide us to select appropriate fairness constraints such that the loss in accuracy is not too much.
We first define the price of fairness constraints.

\begin{definition}[\textbf{Price of fairness constraints}]
	\label{def:price}
	Given $\ell,u\in \R^{p}$, define $f_{\ell,u}^\star$ to be an optimal fair classifier for \ref{eq:progfair}.
	The \emph{price of fairness constraints} is defined by
	\[
	F(\ell,u)=\Pr_\Im\left[f^\star_{\ell,u}\neq Y\right]-\Pr_\Im\left[f^\star_{0,1}\neq Y\right],
	\] 
	i.e., the increase of the prediction error by introducing fairness constraints.
\end{definition}

\noindent
Since requiring tighter fairness constraints can never improve the performance of the optimal fair classifier, $F: \R^{2p}\rightarrow \R_{\geq 0}$ is non-decreasing as $\ell_i$ increases or $u_i$ decreases $(i\in [p])$.

For simplicity, we again consider the case of $\calF=\left\{0,1\right\}^{\calX}$, a single sensitive attribute and a single group performance function $q\in \calQ_{\mathrm{lin} }$.
Similar to Appendixs~\ref{app:algorithm} and~\ref{app:multi}, the following theorem can be extended to $\calF=\left\{0,1\right\}^{\calX\times [p]}$ and generalized to multiple sensitive attributes and multiple group performance functions $q^{(i)}\in \calQ_{\mathrm{linf} }$.

\begin{theorem}[\textbf{Price of fairness constraints}]
	\label{thm:price}
	Suppose $\calF:=\left\{0,1\right\}^\calX$ and $q\in \calQ_{\mathrm{lin} }$.
	For any $x\in \calX$, define $\eta(x):=\Pr\left[Y=1\mid X=x\right]$.
	Given $\ell,u\in [0,1]^p$, define $\lambda^\star$ to be the optimal solution of~\eqref{eq:main_lambda} and $s_{\lambda^\star}$ as in~\eqref{eq:s_lambda}.
	%
	Then 
	$$
	F(\ell,u)=2\cdot \Exp_X\left[B_{\lambda^\star}\left(X\right)\right],
	$$ 
	where
	$$
	B_{\lambda^\star}\left(x\right) = \left|\eta(x)-c\right|\cdot \I\left[\left(\eta(x)-0.5\right)\cdot s_{\lambda^\star}(x)\leq 0 \right].
	$$
\end{theorem}

\begin{proof}
	By Lemma~\ref{lm:not_sensitive_find_lambda}, there exists some $\lambda^\star$ of the form as stated in the theorem, such that $\CS(f^\star_{\ell,u}; 0.5)= \CS(\I\left[s_{\lambda^\star}>0\right]; 0.5) $, and $f^\star_{0,1} = \CS(\I\left[\eta-0.5>0\right]; 0.5 )$.
	Recall that $\Pr\left[f\neq Y\right]=2\cdot\CS(f;0.5)$ for any $f$.
	Then by Lemma~\ref{lm:costrisk}, 
	\begin{align*}
	& \frac{1}{2}F(\ell,u) =  \CS(f^\star_{\ell,u};0.5)-\CS(f^\star_{0,1};0.5) \\
	= & \CS(\I\left[s_{\lambda^\star}>0\right];0.5)-\CS(\I\left[\eta-0.5>0\right];0.5) \\
	= & (1-0.5)\cdot \pi+ \Exp_X\left[\left(0.5-\eta(X)\right)\cdot \I\left[s_{\lambda^\star}(X)>0\right]\right] \\
	& - (1-0.5)\cdot \pi- \Exp_X\left[\left(0.5-\eta(X)\right)\cdot \I\left[\eta(X)-0.5>0\right]\right] \\
	= & \Exp_X\left[\left(0.5-\eta(X)\right)\cdot \left(\I\left[s_{\lambda^\star}(X)>0\right]-\I\left[\eta(X)-0.5>0\right]\right)\right] \\
	= & \Exp_X\left[|\eta(X)-0.5|\cdot \I\left[(\eta(X)-0.5)\cdot s_{\lambda^\star}(X)\leq 0\right]\right] \\
	= & \Exp_X\left[B_{\lambda^\star}\left(X\right)\right].
	\end{align*}
	The second last equality can be verified as follows. 
	For any fix $x\in \calX$, consider two terms 
	\[
	\alpha:=\left(0.5-\eta(x)\right)\cdot \left(\I\left[s_{\lambda^\star}(x)>0\right]-\I\left[\eta(x)-0.5>0\right]\right),
	\] 
	and
	\[
	\beta:=|\eta(x)-0.5|\cdot \I\left[(\eta(x)-0.5)\cdot s_{\lambda^\star}(x)\leq 0\right].
	\] 
	\begin{enumerate}
		\item If $\eta(x) = 0.5$, then $\alpha = \beta = 0$.
		\item If $\eta(x)-0.5 >0$ and $s_{\lambda^\star}(x)>0$, then $\alpha = \beta = 0$.
		\item If $\eta(x)-0.5 >0$ and $s_{\lambda^\star}(x)\leq 0$, then $\alpha = \beta = \eta(x)-0.5$.
		\item  If $\eta(x)-0.5 <0$ and $s_{\lambda^\star}(x)\geq 0$, then $\alpha = \beta = c-\eta(x)$.
		\item If $\eta(x)-0.5 < 0$ and $s_{\lambda^\star}(x)< 0$, then $\alpha = \beta = 0$.
	\end{enumerate}
\end{proof}

\noindent
Observe that if $\lambda^\star=0$, then $B_{\lambda^\star}\left(x\right) = 0$ for any $x\in \calX$.
It implies that there is no price of fairness, i.e., $F(\ell,u)=0$.
By the definition of $s_{\lambda^\star}$, observe that if $\eta(x)-0.5$ and 
$$\eta'(x):=\sum_{i\in [p]}\lambda^\star_i \cdot \left(\sum_{j=1}^{k} \frac{\alpha^{(i)}_j}{\Pr\left[G_i,\event^{(i)}_j \right]}\cdot \Pr\left[G_i,\event^{(i)}_j\mid X=x\right]\right)$$ have the same sign, then $B_{\lambda^\star}(x)=0$.
Hence, if two random variables $\eta(X)-0.5$ and $\eta'(X)$ have the same sign with high probability, then the price of fairness $F(\ell,u)$ is small.
Thus, $F(\ell,u)$ depends on the alignment between the target variable $\eta(X)-0.5$ and the variable $\eta'(X)$ relating to the sensitive attribute.

\section{Other Experiments}
\label{sec:other_experiments}
For German and COMPAS datasets, we present here the plots for accuracy vs output fairness for Algo~\ref{alg:meta_mult}-SR and Algo~\ref{alg:meta_mult}-FDR and other algorithms. 

\begin{figure}
		\centering
		\includegraphics[width=1\linewidth]{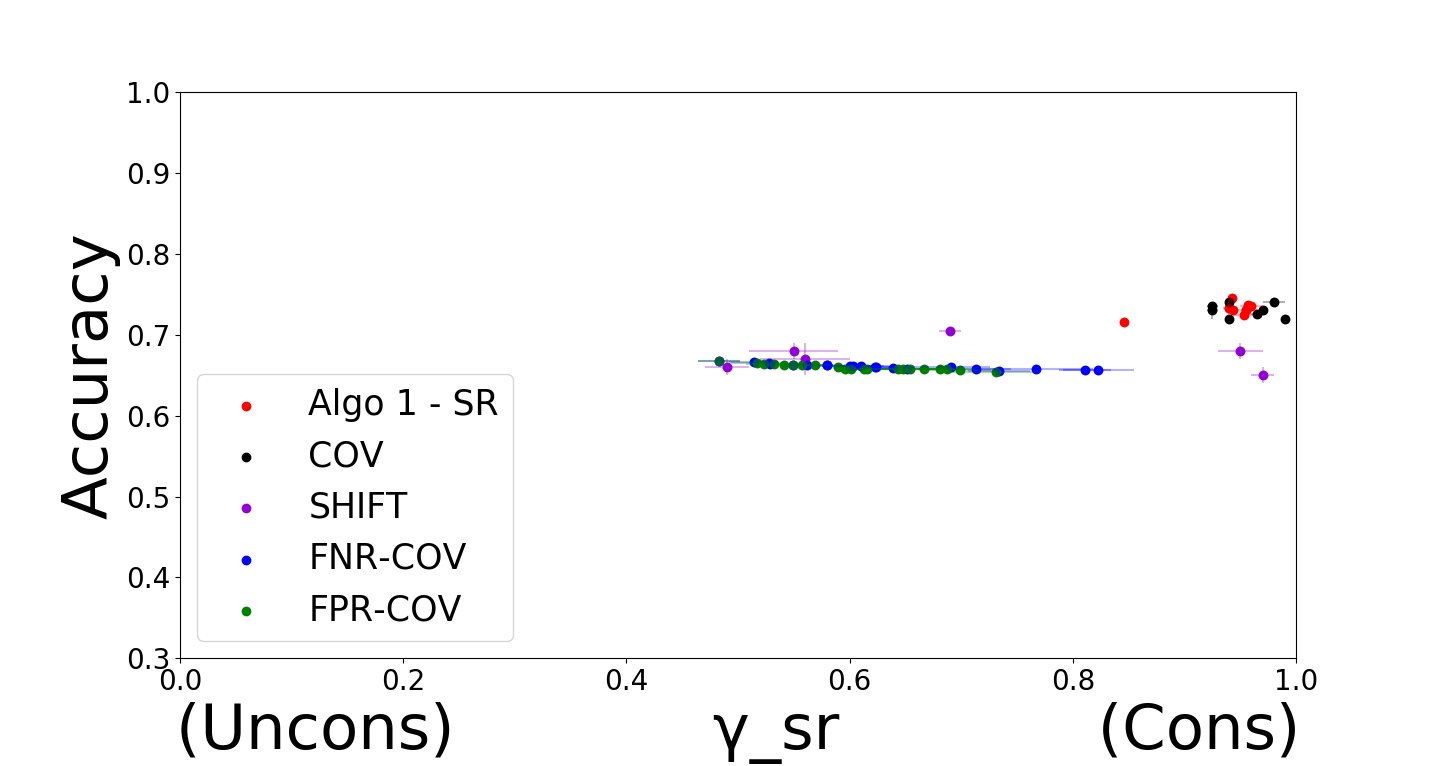}
		\caption{Acc. vs. $\gamma_{\rm sr}$. Algo~\ref{alg:meta_mult}-SR can achieve better fairness with respect to SR  for German dataset and is indistinguishable with respect to accuracy.}
		\label{fig:german_sr}
\end{figure}
\begin{figure}
		\centering
			\includegraphics[width=1\linewidth]{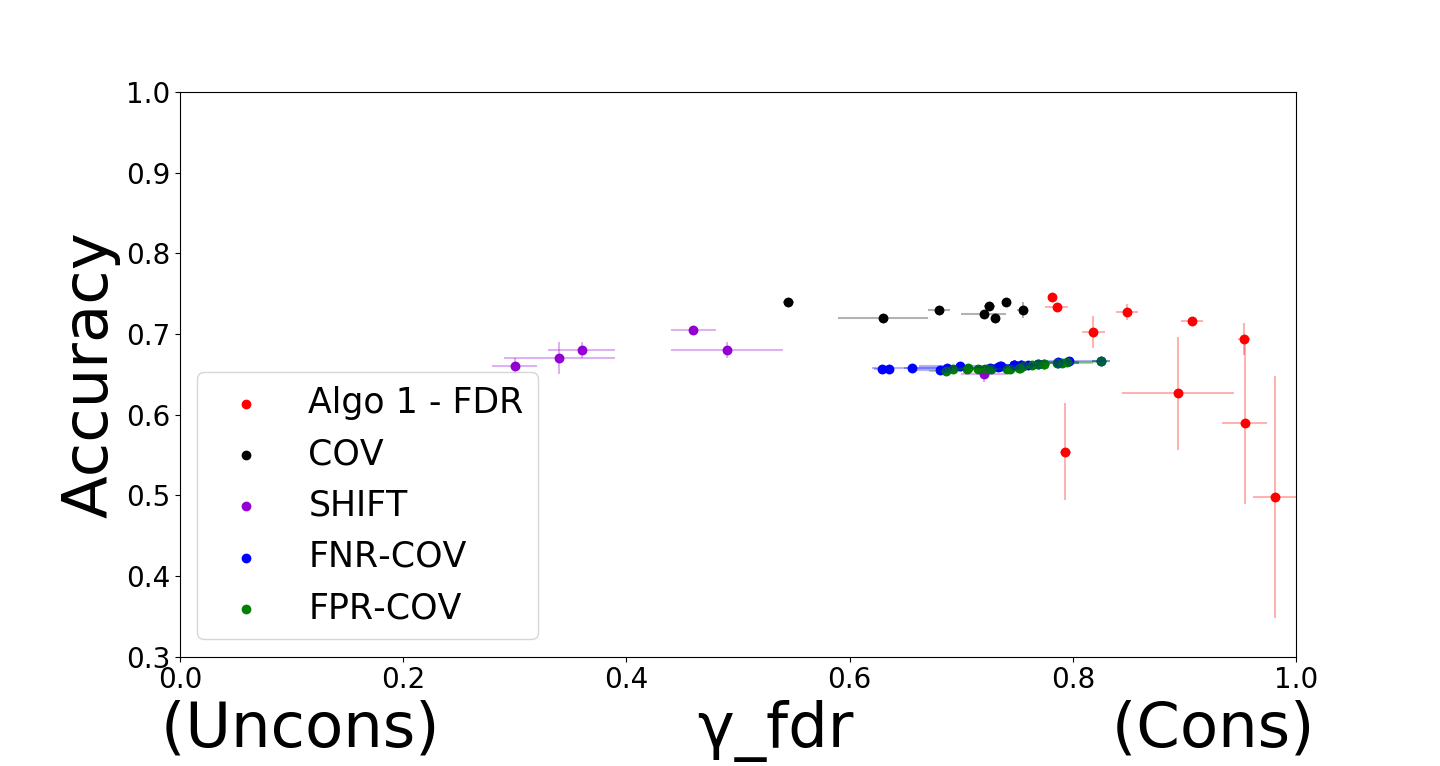}
		\caption{Acc. vs. $\gamma_{\rm fdr}$. Algo~\ref{alg:meta_mult}-FDR Algo~\ref{alg:meta_mult}-FDR achieves better fairness with respect to FDR  for German dataset and is indistinguishable with respect to accuracy.}
		\label{fig:german_fdr}
\end{figure}

\begin{figure}

		\centering
		\includegraphics[width=1\linewidth]{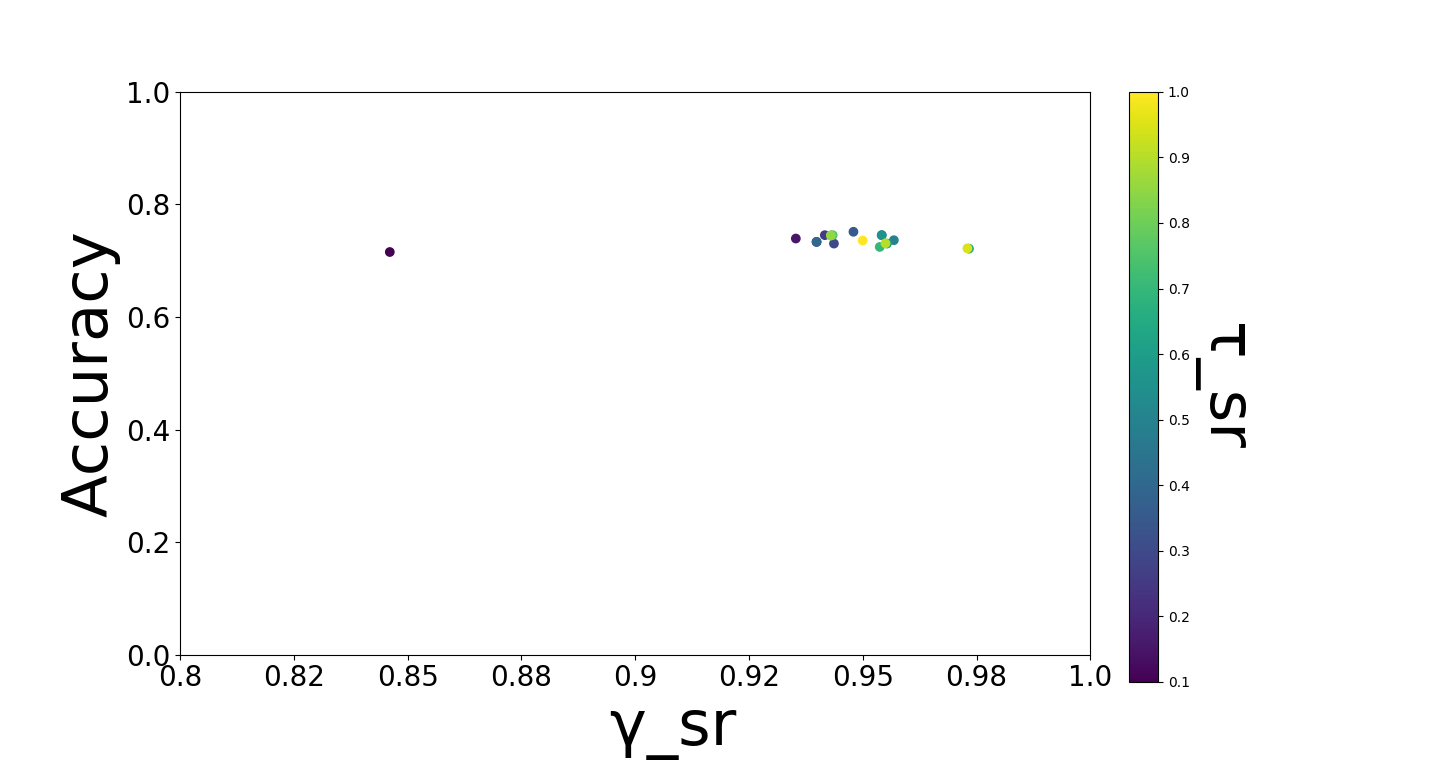}
		\caption{Acc. vs. $\gamma_{\rm sr}$ in Algo~\ref{alg:meta_mult}-SR for different values of input $\tau_{\rm sr}$ for German dataset.}
		\label{fig:german_sr_tau}
	%
\end{figure}

\begin{figure}

		\centering
			\includegraphics[width=1\linewidth]{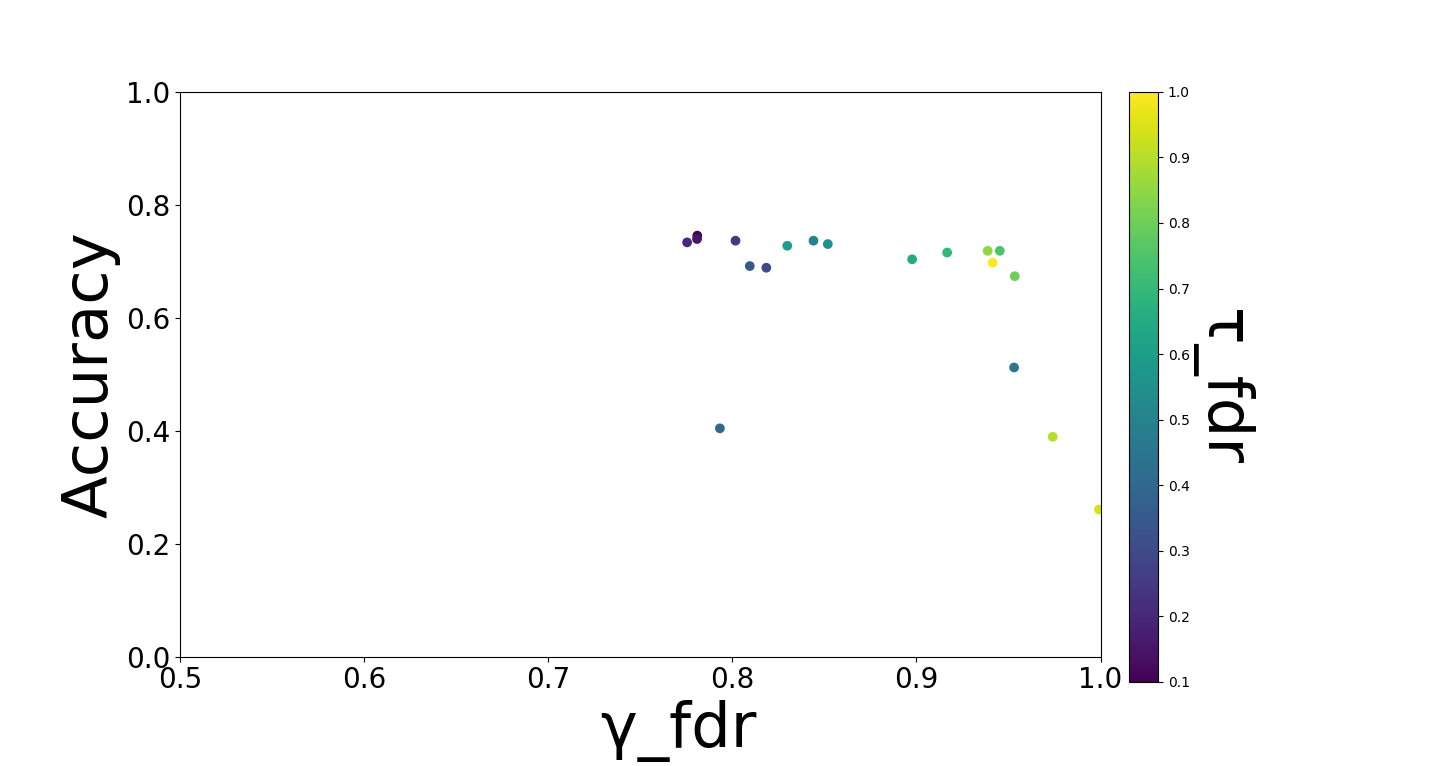}
		\caption{Acc. vs. $\gamma_{\rm fdr}$ in Algo~\ref{alg:meta_mult}-FDR for different values of input $\tau_{\rm fdr}$ for German dataset.}
		\label{fig:german_fdr_tau}
\end{figure}

\begin{figure}
		\centering
		\includegraphics[width=1\linewidth]{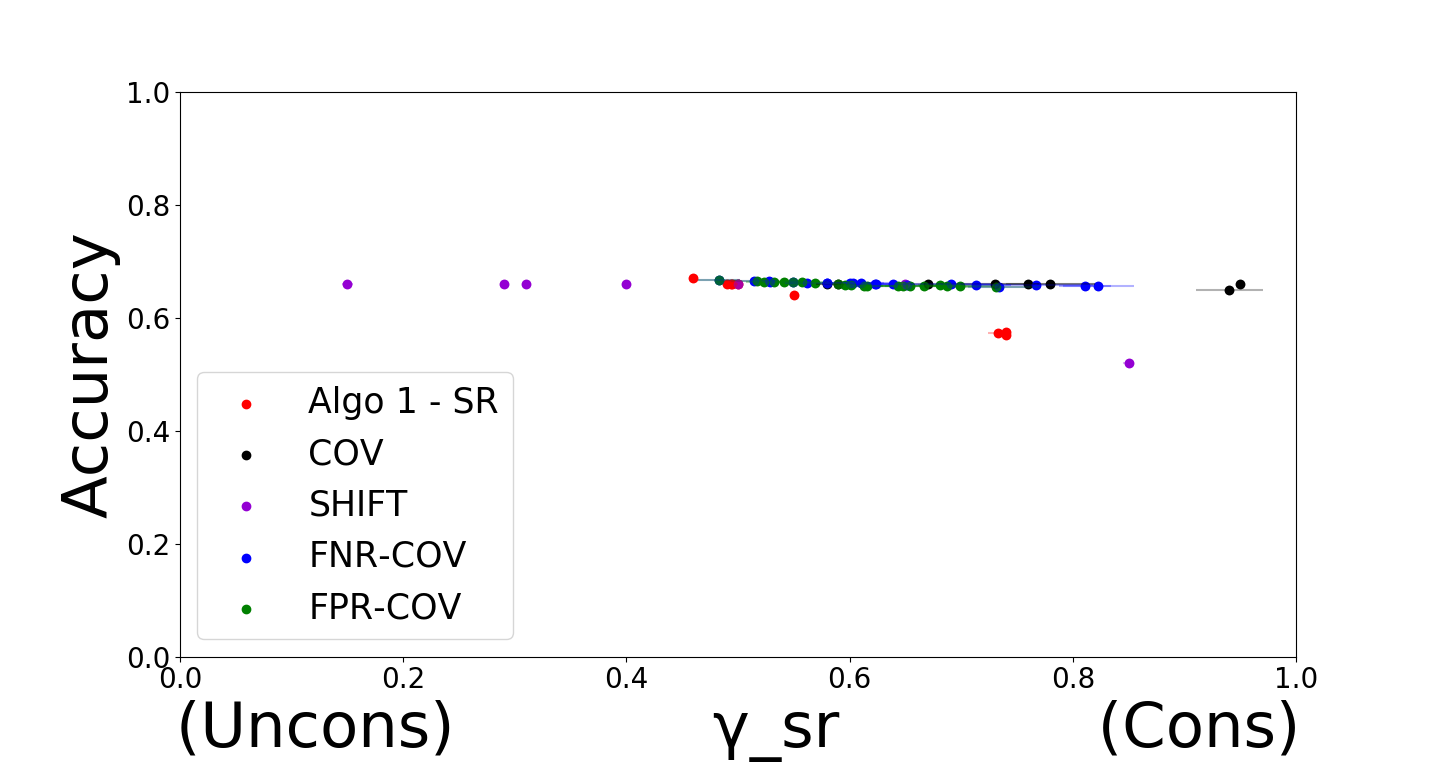}
		\caption{Acc. vs. $\gamma_{\rm sr}$. Algo~\ref{alg:meta_mult}-SR can achieve better fairness with respect to SR than any other method for COMPAS dataset, albeit at a loss to accuracy.}
		\label{fig:compass_sr}
\end{figure}

\begin{figure}
		\centering
			\includegraphics[width=1\linewidth]{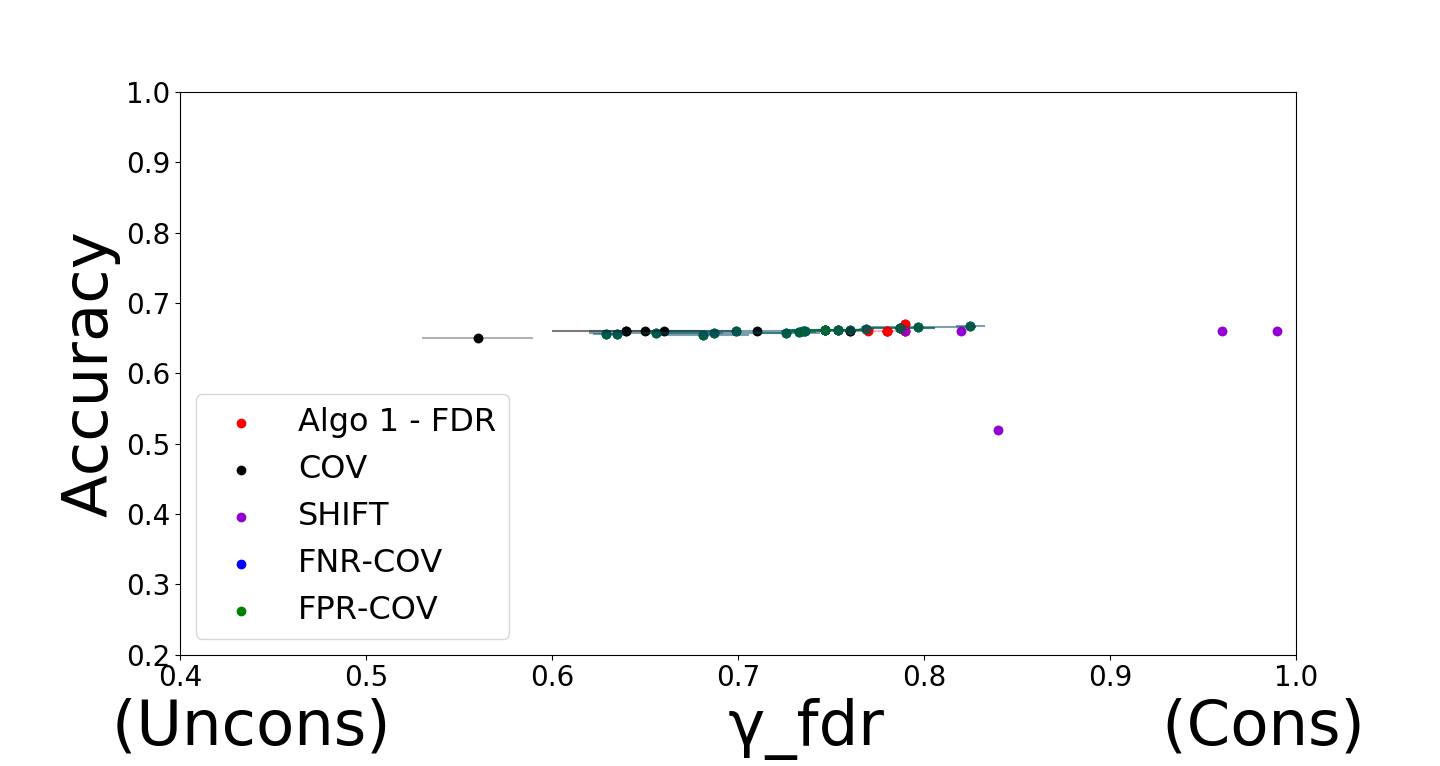}
		\caption{Acc. vs. $\gamma_{\rm fdr}$. Algo~\ref{alg:meta_mult}-FDR Algo~\ref{alg:meta_mult}-FDR achieves better fairness with respect to FDR  for COMPAS dataset and is indistinguishable with respect to accuracy.}
		\label{fig:compass_fdr}
\end{figure}

\begin{figure}
		\centering
		\includegraphics[width=1\linewidth]{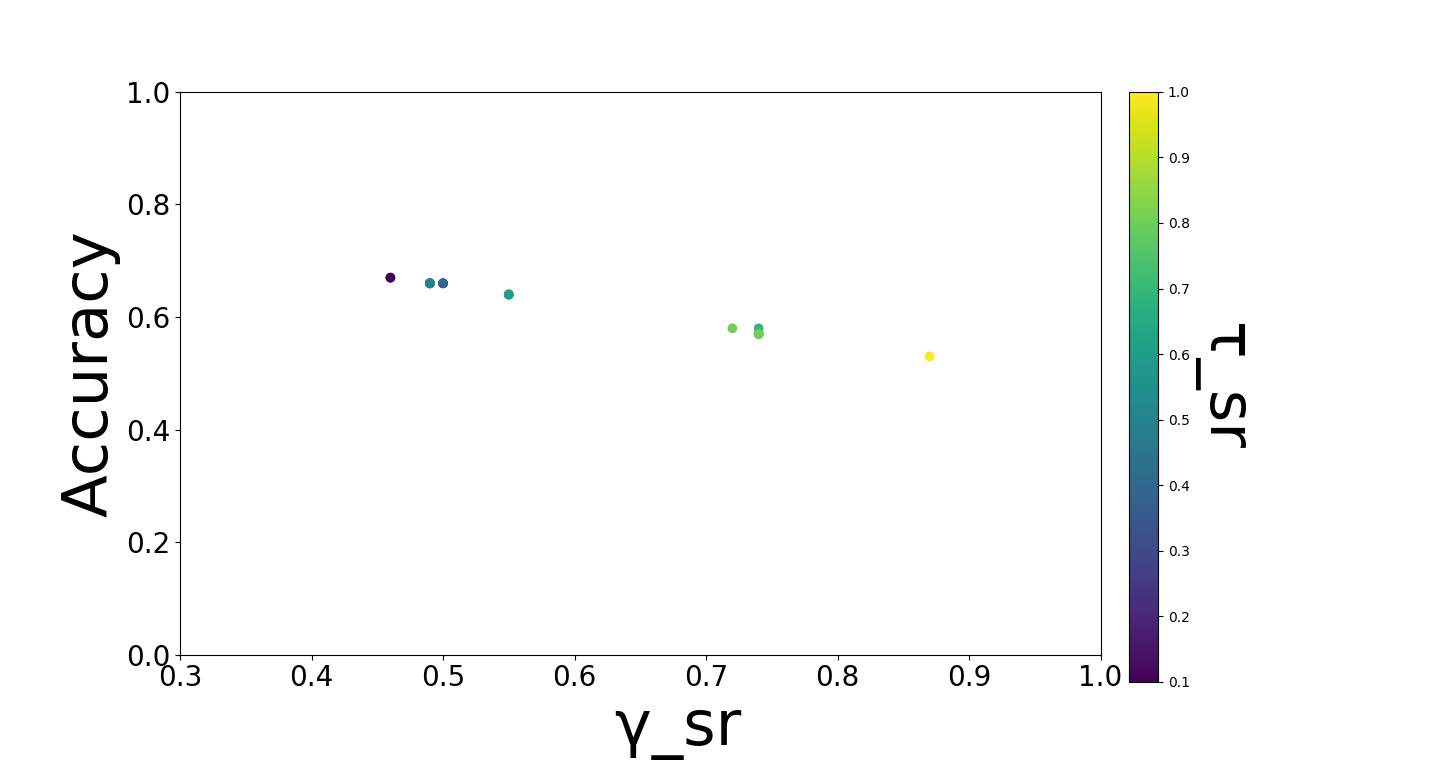}
		\caption{Acc. vs. $\gamma_{\rm sr}$ in Algo~\ref{alg:meta_mult}-SR for different values of input $\tau_{\rm sr}$ for COMPAS dataset.}
		\label{fig:compass_sr_tau}
\end{figure}

\begin{figure}
		\centering
			\includegraphics[width=1\linewidth]{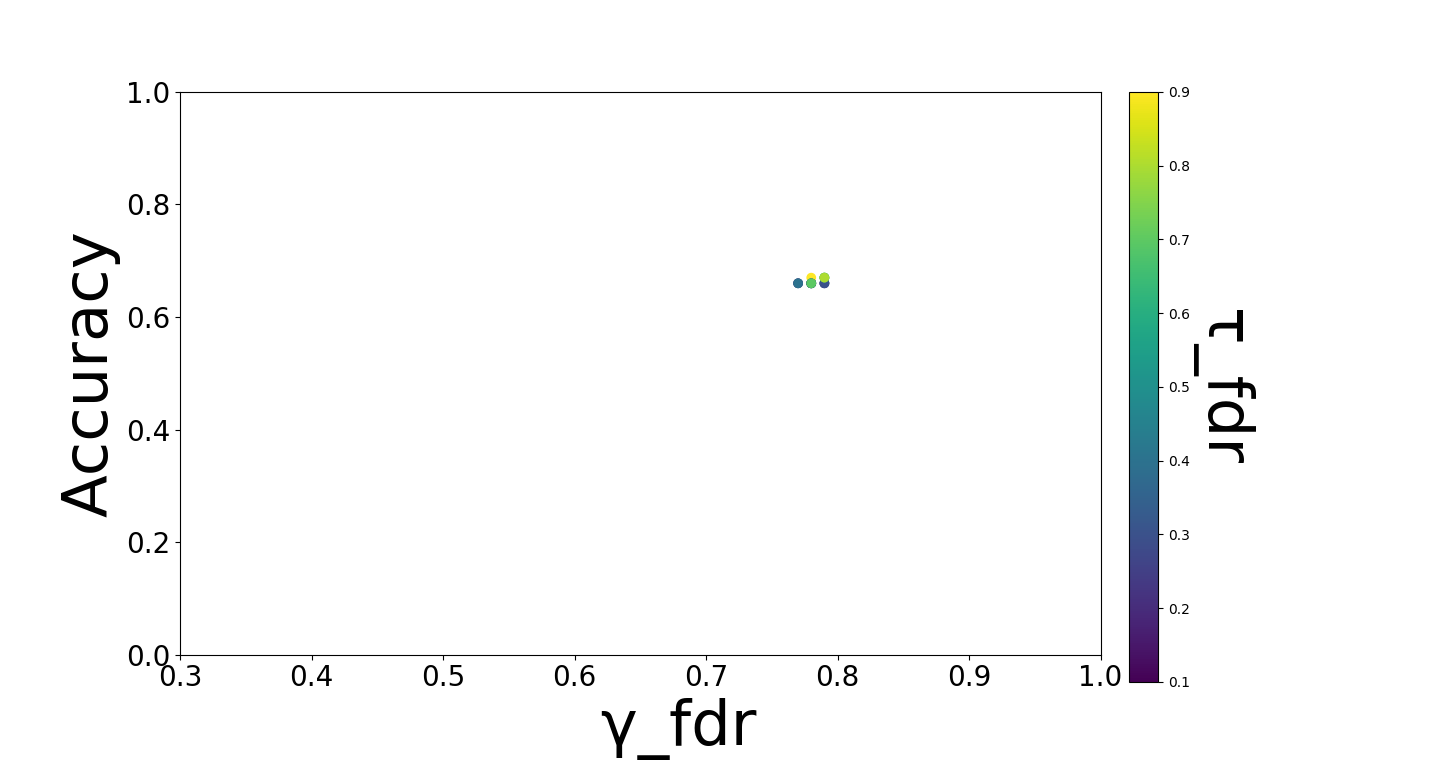}
		\caption{Acc. vs. $\gamma_{\rm fdr}$ in Algo~\ref{alg:meta_mult}-FDR for different values of input $\tau_{\rm fdr}$ for COMPAS dataset.}
		\label{fig:compass_fdr_tau}
\end{figure}

\end{document}